\documentclass{article}

\usepackage{fullpage}
\usepackage[utf8]{inputenc} % allow utf-8 input
\usepackage[T1]{fontenc}    % use 8-bit T1 fonts
\usepackage{hyperref}       % hyperlinks
\usepackage{url}            % simple URL typesetting
\usepackage{booktabs}       % professional-quality tables
\usepackage{amsfonts}       % blackboard math symbols
\usepackage{nicefrac}       % compact symbols for 1/2,etc.
\usepackage{microtype}      % microtypography
\usepackage{graphicx}
\usepackage{amsmath}
\usepackage{bm}
\usepackage{amsthm}
\usepackage{mathtools}
\usepackage[mathscr]{euscript}
\usepackage{bbm}
\usepackage{float}
\usepackage{subfig}
\usepackage{natbib}

\DeclareMathOperator*{\E}{\mathbb E}

\DeclareMathOperator*{\argmin}{argmin}

\DeclareMathOperator{\1}{\mathbbm 1}

\DeclareMathOperator*{\Reg}{\mathsf{Reg}}

\newcommand{\cA}{\mathcal{A}}
\newcommand{\cC}{\mathcal{C}}

\newcommand{\cG}{\mathcal{G}}

\newcommand{\sA}{{\mathscr A}}

\newcommand{\sC}{{\mathscr C}}

\newcommand{\sG}{{\mathscr G}}

\newcommand{\sK}{{\mathscr K}}

\newcommand{\sS}{{\mathscr S}}

\newcommand{\B}{B}
\newcommand{\h}{\widehat}

\newcommand{\e}{\epsilon}
\newcommand{\set}[2][]{#1 \{ #2 #1 \} }
\newcommand{\ignore}[1]{}

\newcommand{\calC}{\mathcal{C}}

\newcommand{\MDP}{\text{MDP}(\sS, \sA, \cC, \bm{\theta})}
\newcommand{\OPT}{\text{OPT}}
\newcommand{\cmax}{c}

\makeatletter
\newtheorem*{rep@theorem}{\rep@title}
\newcommand{\newreptheorem}[2]{%
\newenvironment{rep#1}[1]{%
 \def\rep@title{#2 \ref{##1}}%
 \begin{rep@theorem}}%
 {\end{rep@theorem}}}
\makeatother

\newtheorem{theorem}{Theorem}
\newreptheorem{theorem}{Theorem}

\newreptheorem{lemma}{Lemma}

\newtheorem{definition}{Definition}
\newtheorem{corollary}{Corollary}

\hypersetup{
  colorlinks   = true,
  urlcolor     = blue,
  linkcolor    = blue,
  citecolor   = blue
}

\title{Beyond Individual and Group Fairness
}
\author{Pranjal Awasthi\thanks{
Google Research and Rutgers University.
\tt{pranjalawasthi@google.com}}
\and 
Corinna Cortes\thanks{
Google Research.
\tt{corinna@google.com}}
\and 
Yishay Mansour\thanks{
Tel Aviv University and Google Research.
\tt{mansour.yishay@gmail.com}}
\and
Mehryar Mohri\thanks{
Google Research and Courant Institute.
\tt{mohri@google.com}}
}

\date{}

\begin{document}

\maketitle

\begin{abstract}
  We present a new data-driven model of fairness that, unlike existing
  static definitions of individual or group fairness is guided by the
  unfairness complaints received by the system. Our model supports
  multiple fairness criteria and takes into account their potential
  incompatibilities. We consider both a stochastic and an adversarial
  setting of our model. In the stochastic setting, we show that our
  framework can be naturally cast as a Markov Decision Process with
  stochastic losses, for which we give efficient vanishing regret
  algorithmic solutions. In the adversarial setting, we design efficient
  algorithms with competitive ratio guarantees. We also report the
  results of experiments with our algorithms and the stochastic framework
  on artificial datasets, to demonstrate their effectiveness
  empirically.
\end{abstract}

\section{Introduction}
\label{sec:intro}

Learning algorithms trained on large amounts of data are increasingly
adopted in applications with significant individual and social
consequences such as selecting loan applicants, filtering resumes of
job applicants, estimating the likelihood for a defendant to commit
future crimes, or deciding where to deploy police officers.  Analyzing
the risk of bias in these systems is therefore crucial. In fact, that
is also critical for seemingly less socially consequential
applications such as ads placement, recommendation systems, speech
recognition, and many other common applications of machine learning.
Such biases can appear due to the way the training data has been
collected, due to an improper choice of the loss function optimized,
or as a result of some other algorithmic choices.
This has motivated a flurry of recent research work on the topic of
\emph{fairness} and \emph{algorithmic bias} in machine learning
\citep{dwork2012fairness, zemel2013, hardt2016equality, kleinberg2017,
  pleiss2017fairness, agarwal2018reductions, KearnsNeelRothWu2018, 
gillen2018online, SharifiMalvajerdiKearnsRoth2019}.
% which we
% cannot survey adequately in such a short space. Instead, we will
% discuss some general trends and will refer the reader to a recent book
% \citep{BarocasHardtNarayanan2020}, and to a more extensive
% discussion of related work in Appendix~\ref{sec:app-related}.

How should fairness be defined? This has been one of the key
challenges faced by most recent publications dealing with the
topic. Two broad families of definitions have been adopted in the
literature: \emph{statistical} or \emph{group fairness}, and
\emph{individual fairness}.
Statistical fairness is typically defined via the choice of some
protected sub-groups, often based on sensitive attributes such as
race, gender, ethnicity, or sexual orientation, and that of a metric
such as \emph{false positive rate}, \emph{false negative rate}, or
\emph{classification error}. The requirement is an equalized metric
for all protected sub-groups. This is by far the most popular
definition of fairness and includes a very wide literature. Some
common examples of group fairness criteria include
\emph{counterfactual} or \emph{demographic parity}
\citep{KusnerLoftusRussellSilva17} and \emph{equality of opportunity}
\citep{hardt2016equality}. The benefits of these metrics is that they
can be tested and a classifier can be learned by imposing equalized
metric constraints. On the other hand, they sometimes admit a trivial
solution with clearly undesirable
properties \citep{KearnsNeelRothWu2018}. Furthermore, there is no
general agreement on the choice of the protected groups considered and
different metrics can be incompatible \citep{kleinberg2017,
  feller2016computer}.

Group fairness only provides an average guarantee for the individuals
in a protected group. In contrast, individual fairness requires that
similar individuals be treated similarly by the model. This similarity
is often defined according to an underlying metric over user features
\citep{zemel2013, dwork2012fairness, joseph2016fairness}.  The
problem, however, is that it is not clear what that metric should be
and there is no general agreement on its definition. Furthermore, the
analysis of individual fairness often resorts to strong functional
assumptions.

\begin{figure}[t]
\centering
\includegraphics[scale = .3]{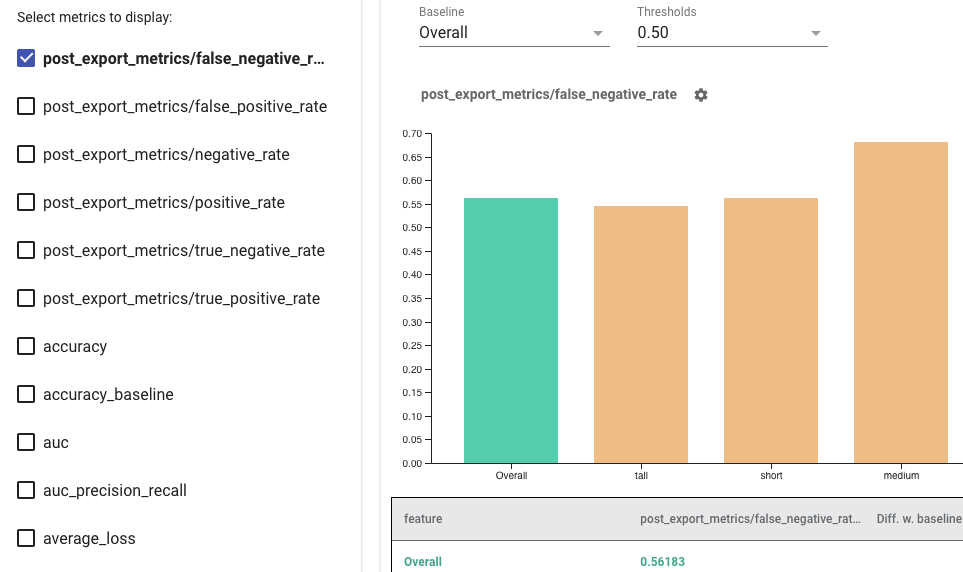}
\caption{A snapshot of the fairness indicator tool recently launched in
  Tensorflow.
\label{fig:fairness-indicators}}
\vskip -.25in
\end{figure}
 
The absence of a unique metric capturing algorithmic fairness is not
just a theoretical obstacle, it can result in troublesome dilemma in
practice.  An illuminating example is the analysis of the COMPAS tool
for predicting recidivism by \cite{angwin2019machine}. The authors
showed that, among black defendants who do not recidivate, the tool
predicted incorrectly at twice the rate than it did for white
defendants who did not recidivate. In other words, the tool was unfair
according to the \emph{false positive rate} metric. The creator of the
tool, Northpointe, responded by demonstrating that the tool was fair
according to other natural measures such as AUC (Area Under the ROC
Curve). Later work showed that this tension is inherent and that it is
often impossible to simultaneously satisfy multiple seemingly natural
fairness criteria \citep{kleinberg2017} (see also discussion by
\cite{feller2016computer}).

Thus, there is no single generally accepted definition of fairness.
Moreover, while algorithms tailored to a specific metric would be
effective at first, experience shows that they become unrealistic over
time: once a system is deployed and it interacts with the environment
and its end-users, hidden biases encoded in the system design emerge,
which in turn raise fairness complaints from new user groups and
metrics originally not accounted for. This suggests working with
multiple fairness criteria. However, as already pointed out, some
criteria cannot be simultaneously satisfied.

To deal with the issues just discussed, we propose a data-driven model
of fairness resolution guided by the unfairness complaints received,
rather than by a single static definition of individual or group
fairness: at each time step, a fairness resolution algorithm chooses
to \emph{fix} a criterion, thereby \emph{unfixing} incompatible
criteria, incurring a fixing cost, as well as some loss due to a
new sequence of fairness complaints received.
The fixing cost depends on the criterion.  For instance, addressing
differences in false positive rates might require augmenting the loss
with a new regularization term, whereas complying with a specific
individual fairness criterion could require collecting more data and
learning an accurate distance metric among individuals.
The objective of the fairness resolution algorithm is to minimize its
cumulative loss over the course of multiple interactions with the
environment.

To illustrate our model, consider the fairness indicator tool recently
launched in TensorFlow (Figure~\ref{fig:fairness-indicators}).  Using
this tool, one can monitor the performance of the current classifier
according to different fairness metrics. As more data is collected and
the system interacts with the environment, the cost incurred by the
system on each metric is updated.  This cost encodes quantitative
measures such as the number of data points violating a metric, as well
as more qualitative ones such as the negative publicity generated as a
result of violating a fairness criterion, or its legal and
ethical ramifications. As these costs are updated, the system designer is faced with a choice of which metrics to prioritize at a particular time. Our goal in this work is to propose a model and algorithmic solutions to make near optimal choices in such scenarios.

In Section~\ref{sec:model}, we define our model in more detail. Our
model supports multiple fairness criteria and takes into account their
potential incompatibilities. We consider both a stochastic and an
adversarial setting. In the stochastic setting
(Section~\ref{sec:stochastic}), we show that our framework can be
naturally cast as a Markov Decision Process with stochastic losses,
for which we give efficient vanishing regret algorithmic solutions. In
the adversarial setting (Section~\ref{sec:adversarial}), we describe
algorithms with competitive ratio guarantees. We also report the
results of experiments  
(Section~\ref{sec:experiments}) with our
algorithms to demonstrate their effectiveness empirically.

\section{Fairness Resolution Model}
\label{sec:model}

We consider the problem of resolving fairness issues in the presence
of multiple fairness criteria. Not all fairness criteria can be
satisfied simultaneously.  The constraints can be specified by an
undirected graph $\sG = (V, E)$, where each vertex represents a fairness
criterion and where an edge between vertices $v_i$ and $v_j$ indicates
that criteria $v_i$ and $v_j$ cannot be simultaneously satisfied. We will
denote by $V = \set{v_1, \ldots, v_k}$ the set of $k$ fairness
criteria considered.  Figure~\ref{fig:g} illustrates these
definitions. Note, that vertices may represent joint criteria as in Figure~\ref{fig:g}(b).

\begin{figure}[t]
\centering
\begin{tabular}{c@{\hspace{2cm}}c}
\includegraphics[scale=.66]{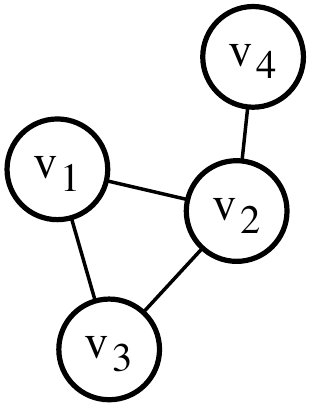} &
\includegraphics[scale=.66]{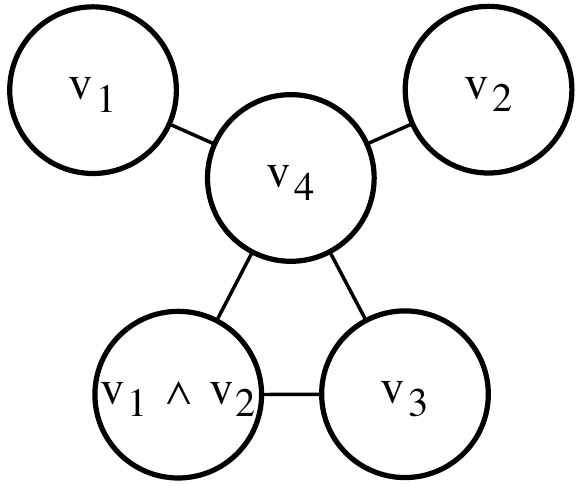}\\
(a) & (b)
\end{tabular}
\caption{(a) Illustration of constraints graph $\sG$.
  $v_1, v_2, v_3, v_4$ represent $4$ different fairness criteria. (b)
  More generally, each vertex can represent a joint fairness
  criterion, for example $v_1 \wedge v_2$. This helps specify joint
  constraints such as the following: $v_1$, $v_2$, and $v_3$ cannot be
  simultaneously satisfied.}
\label{fig:g}
\vskip -.15in
\end{figure}

At any time, a vertex $v_i$ can be either \emph{fixed}, meaning that
criterion $v_i$ is met or is not violated, or \emph{unfixed}, meaning
the opposite.  \emph{Fixing} criterion $v_i$ entails an algorithmic
and resource allocation cost that we denote by $c_i$.  Depending on
the type of criterion and intervention, $c_i$ may include different
costs, such as that of additional data collection, the average number
of human hours needed to address the fairness violation, or the loss
incurred in some metric, such as accuracy.  \ignore{An edge $(v_i, v_j) \in E$
in the graph indicates that vertices $v_i$ and $v_j$ cannot be both
fixed at any time.} Initially, all vertices are in an unfixed state. At each time step, a
fairness resolution system or algorithm selects some action, which may
be to fix some unfixed vertex $v_i$, thereby incurring the cost $c_i$
and \emph{unfixing} any vertex adjacent to $v_i$, or the algorithm may select the null action, not to fix
or unfix any vertex, and wait to collect more data. 

In response to its action, the
system receives a sequence of fairness complaints.  The complaints
affect one or several vertices of $\sG$ and result in a fairness loss
corresponding to the vertices affected.
The objective of the algorithm is to minimize the total cost incurred
over a period of time, which includes the total fairness loss accrued
as well as the total cost of fixing various fairness criteria over
that time period.

As an example, in the context of the COMPAS controversy discussed in
the previous section, fixing the criteria corresponding to the
false positive metric may result in an increase of the complaints related
to say a calibration  metric. The decision to fix a fairness
criterion may also have positive implications for other metrics. For
instance, criteria such as the false positive rate are correlated with
accuracy and hence fixing one can be expected to decrease the
complaints for the other.

\textbf{Realizability.} 
While the focus of our study is mainly theoretical and algorithmic, we
wish to emphasize that our model can indeed be realized in practice.
Graph $\sG$ can be derived from analyzing past fairness complaints and
by measuring how fixing one criterion affects the performance on
others.  The assignment of a complaint to one or more fairness
criteria can be achieved by making use of known unfair classifiers, as discussed in Section~\ref{sec:intro}, or via a multi-class multi-label classifier
trained on past data. Finally, the average fixing cost specific to
each criterion can be estimated from past experience.
Our model also provides the flexibility of accounting for
incompatibilities among more than two criteria such as those discussed
by \cite{kleinberg2017} and \cite{feller2016computer}.  This can be
achieved by augmenting the graph with vertices representing joint
criteria as in Figure~\ref{fig:g}(b). The graph in that example
stipulates in particular that $v_1$, $v_2$ and $v_3$ cannot be all
simultaneously satisfied.

\ignore{
As an example,
consider a subset of $m$ criteria $C = \set{v_{i_1}, \dots, v_{i_m}}$
that cannot be simultaneously satisfied. In order to model this we add
$m$ additional criteria vertices $u_1, \dots, u_m$ to the graph
$\sG$. Vertex $u_t$ captures the subset $C \setminus \{v_{i_t}\}$ as its
children and shares the same neighbors as $v_{i_t}$. Furthermore, we
add edges between $u_{t}$ and $v_{i_t}$ for all $t$ and also edges
between each pair of vertices $u_{t_1}$ and $u_{t_2}$. Finally, we
modify the state transitions to enforce that fixing a vertex $u_t$
fixes all its children and unfixing $u_t$ unfixes a certain
pre-specified subset, or all of the children of $u_t$
}

\textbf{Ethical implications.}  It is worth discussing various
aspects of our model and questioning its social implications, in
particular its potential impact on social values of fairness and
equity.
There is no generally agreed upon definition of these terms, let alone
a computational one.  A key motivation behind the design of our model
is precisely to refrain from proposing yet another definition of
fairness, accepted by some, rejected by others. Indeed, experience
shows that the notion of fairness is difficult to define.  No two
individuals seem to share the same notion of fairness, perhaps because
of their distinct personal interests. Similarly, definitions of group fairness
favoring some protected groups seem not to be agreed upon by other
social groups. Additionally, hidden unfairness effects have been shown
to come up as a result of seemingly natural notions of group
fairness. Moreover, while discrimination based on given sensitive attributes is illegal by law in many countries, the notion of \emph{protected group} is not well
defined.  In fact, in practice, reactions to a deployed software
system reveal new social groups defined by more complex attributes
than standard protected groups defined via standard attributes such as race, gender,
ethnicity, or income level.

%Instead, we suggest a data-driven model of fairness guided by the
%unfairness complaints received by the system. Thus, 
Instead of a
\emph{static definition} of fairness, we advocate a \emph{dynamic
  definition} determined by user reactions to the system. This is
further motivated by the fact that complying with multiple fairness
criteria might be impossible.
There may be a better chance, however, for abiding with multiple criteria over time. Our model thus avoids committing to a single dogmatic
definition.  However, it is also subject to some drawbacks. First, a
static definition of fairness may be more convenient from the point of
view of regulators. Second, while we seek not to commit to a single
notion of fairness, we are in fact relying on multiple fairness
criteria, which may include those typically adopted in the past. It is our
hope though that the use of multiple criteria can help limit hidden
biases and that, by virtue of taking into consideration the reactions
to the system, our model is more \emph{democratic} or flexible, and thus a more
suitable candidate for regulations too.
% In the section "Broader Impact" at the end of the paper, 
% we discuss various aspects of our model and its social implications.

In the next sections, we study the computational and algorithmic
aspects of our model both in the stochastic and the adversarial
setting.  We will present nearly optimal algorithmic solutions for
both settings. Our analysis will further demonstrate the flexibility
of our model.

\section{Stochastic Setting}
\label{sec:stochastic}

In this section, we discuss a stochastic setting of our model that can
be naturally described in terms of a Markov Decision Process (MDP).
Next, we present algorithms with strong regret guarantees for this
setting.

\subsection{Description}

A key observation in this scenario is that, at any time, the
distribution of fairness complaints received by the system is a
function of its current \emph{state}, that is the current set of fixed
or unfixed criteria $v_i$.
Thus, we consider an MDP with a state space
$\sS \subseteq \set{0, 1}^k$ representing the set of bit vectors for
criteria: a state $s \in \set{0, 1}^k$ is defined by $s(i) = 0$ when
criterion $v_i$ is unfixed and $s(i) = 1$ when it is fixed. By
definition of the incompatibility graph $\sG$, $s$ is a valid state iff
the set of fixed criteria at $s$ form an independent set of $\sG$.
When in state $s \in \sS$, the system incurs a
loss $\ell_i^s$ due to complaints related to criterion $i \in
[k]$. $\ell_i^s$ is a random variable assumed to take values in
$[0, \B]$ with mean $\mu_i^{s}$.

The set of actions for our MDP is $\sA = \set{0, 1, \ldots, k}$ where
a non-zero action $i$ corresponds to fixing criterion $i$, while
action $0$ is the null action, that is no criterion is fixed.
Transitions are deterministic: given state $s$ and action $i \in \sA$,
the next state is $s$ if $i = 0$ since the fixed-unfixed bits for
criteria are unchanged; otherwise, for $i \neq 0$ the next state is
the state $s'$ that only differs from $s$ by $s'(i) = 1$ and
(possibly) $s'(j) = 0$ for all $j \in N(i)$, where $N(i)$ denotes the
set of neighbors of $v_i$ in $\sG$, since vertices neighboring $i$ must
be unfixed once $i$ is fixed.

Each action $a = i$ admits a fixing cost $c_i$.  The cost for the null action
is $c_0 = 0$. The loss incurred by the algorithm when taking
action $a$ at state $s$ is the sum of the fixing cost $c_a$ and the
complaint losses at the (possibly) next state $s'$:
$\lambda(s,a) = c_a + \sum_{i = 1}^k \ell_i^{s'}$. The expected loss of transition
$(s, a, s')$ is thus:
\begin{equation}
\label{eq:loss-definition}
\E\left[ c_a + \sum_{i = 1}^k \ell_i^{s'} \right] = c_a + \sum_{i = 1}^k \mu_i^{s'}.
\end{equation}
Note, $c_a$ and the losses $\ell_i^{s'}$ are observed by the algorithm, but
the mean values $\mu_i^{s'}$ are unknown. To keep the formalism simple we assume that the cost $c_a$ of taking an action $a$ is independent of the current state $s$. However, our theoretical results easily extend to the setting where taking an action in different states has different costs.
Figure~\ref{fig:mdp-transitions} illustrates our stochastic MDP model
in the special case of a fully connected graph $\sG$, that is one where
all three criteria are mutually incompatible.

\begin{figure}[t]
  \centering
  \includegraphics[width=0.66\textwidth]{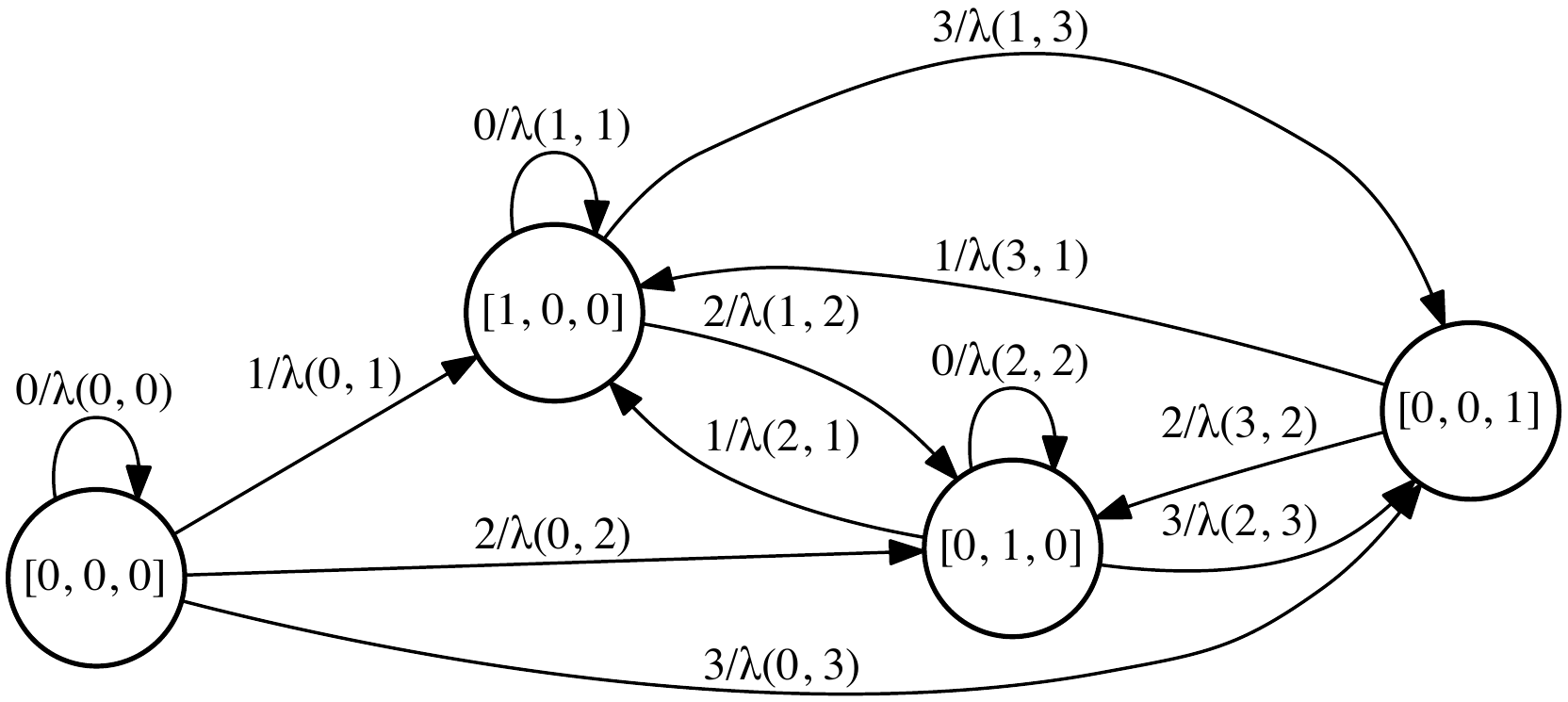}
  \caption{Illustration of the MDP for a fully connected
    incompatibility graph $\sG$ over three criteria. The state set is
    $\sS = \set{\bm{0} = [0,0,0], \bm{1} = [1,0,0], \bm{2} = [0,1,0],
      \bm{3} = [0,0,1]}$ and the action set
    $\sA = \set{0, 1, 2, 3}$. Each transition is labeled with
    $a/\lambda(s,a)$, where $a$ is the action taken from state $s$ and
    where $\lambda(s,a)$ is the total loss incurred as a result of the action.}
    \label{fig:mdp-transitions}
\end{figure}

\ignore{
at state $s$ will
only depend on local correlations involving small subsets of other
vertices.

The above formulation has exponentially
(in the size of the graph) many parameters corresponding to each pair
of states and vertices. However, in practical settings

}

\paragraph{Correlation sets.} In practice, the distribution of
complaints related to a criterion $v_i$ at two different states may be
related. To capture these correlations in a general way, we assume
that a collection $\cC$ of \emph{correlation sets}
$\cC = \set{\sC_1, \sC_2, \dots, \sC_n}$ is given, where each $\sC_j$
has size at most $m$. Notice that the number of sets in $\cC$ is at
most $\binom{k}{\leq m}$.  Each set $\sC_j$ is a subset of the set of
criteria $\set{1, \ldots, k}$ and results from local dependencies among
the criteria it includes. We assume that at a given state $s$, each
set $\sC_j$ generates losses with mean value $\theta^s_j$ per vertex, and that if
two states $s$ and $s'$ admit the same configuration for the vertices
in $\sC_j$, then they share the same parameter
$\theta^s_j = \theta^{s'}_j$. Given a criterion $i$ and a state $s$, we
assume that the loss incurred by criterion $i$ equals the sum of the
individual losses due to each correlation set $\sC_j$ that contains
$i$. Thus, $\mu^s_i$ can be expressed as follows:
\begin{equation}
\label{eq:loss-definition-via-correlation}
\mu^s_i = \sum_{j = 1}^n \theta^s_j \1(i \in \sC_j).
\end{equation}
Since for each $j \in [n]$, there are at most $2^m$ different
configurations for the vertices of $\sC_j$ in a state $s$, there are
at most $2^m n$ distinct parameters $\theta^s_j$.
Let $\bm{\theta}$ denote the vector of all distinct parameters
$\theta^s_j$. Our MDP model can then be denoted
$\text{MDP}(\sS, \sA, \cC, \bm{\theta})$.

\subsection{Algorithm}

We consider an online algorithm that at each time $t$ takes action
$a_t$ from state $s_t$ of the MDP previously described and reaches
state $s_{t + 1}$, starting from the initial state $(0, \ldots,
0)$. For standard MDP settings, the objective of an algorithm can be
formulated as that of learning a policy, that is a mapping
$\pi\colon \sS \to \sA$, with a value as close as possible to that of
the optimal. Here, we are mainly interested in the cumulative
loss of the algorithm over the course of $T$ interactions with the
environment. Thus, the objective of an algorithm $\cA$ can be
formulated as that of minimizing its pseudo-regret, defined by
\begin{equation}
\label{eq:pseudo-regret}
\Reg(\cA) 
= \sum_{t=1}^T \E \Big[ \lambda_t(s_t, a_t) \Big] - \sum_{t=1}^T \E \Big[ \lambda_t (s^{\pi^*}_t, \pi^*(s^{\pi^*}_t)) \Big],
\end{equation}
where $\lambda_t(s, a)$ is the total loss incurred by taking action
$a$ at state $s$ at time $t$ and where $s_1 = (0, \ldots, 0)$ and
$\pi^*$ is the optimal policy.  The expectation is over the random
generation of the complaint losses.
Given the correlation sets and the parameter $\bm{\theta}$, the
optimal policy $\pi^*$ corresponds to moving from the initial state
$(0, \ldots, 0)$ to the state $s^* \in \sS$ with the most favorable
distribution and remaining at $s^*$ forever. We define by $g(s)$ the expected (per time step) loss incurred by staying in state $s$, i.e.,
\begin{align}
    \label{eq:def-gs}
    g(s) \coloneqq \sum_{i = 1}^k 
\mu^s_i.
\end{align}
Then given the parameters of the MDP the optimal state $s^*$ is defined as
follows:
% \begin{equation}
% \label{eq:opt-state}
% s^* = \argmin_{s \in \sS} g(s) \coloneqq \sum_{i = 1}^k 
% \mu^s_i.
% \end{equation}
\begin{equation}
\label{eq:opt-state}
s^* = \argmin_{s \in \sS} g(s) .
\end{equation}

Note, in this definition of $s^*$, we are disregarding the one-time
cost of moving to a state from the initial state, since in the long run the expected cost incurred by staying at a given state governs the choice of the optimal state.

\ignore{
The objective of the algorithm is to learn for the fairness
resolution MDP previously described a policy that admits small regret
after $T$ time steps, where a policy $\pi\colon \sS \to \sA$ is a mapping
from states to actions.

Given the correlation sets and the parameter $\bm{\theta}$, the
optimal policy $\pi^*$ corresponds to moving from the initial state
$(0, \ldots, 0)$ to the state $s^* \in \sS$ with the most favorable
distribution and remaining at $s^*$ forever. State $s^*$ is defined as
follows:
\begin{equation}
\label{eq:opt-state}
s^* = \argmin_{s \in \sS} g(s) \coloneqq \sum_{i = 1}^k 
\mu^s_i.
\end{equation}
Here, we are disregarding the one-time cost of moving to a state from
the initial state.

Let $(s_t, a_t)$ be the state action pairs that result
from running an online algorithm $\cA$ for the above MDP and let
$(s^{\pi^*}_t, a^{\pi^*}_t)$ be the state action pairs that result
from playing the optimal policy $\pi^*$. Then we define the
pseudo-regret of the algorithm as:
\begin{equation}
\label{eq:pseudo-regret}
\Reg(\cA) 
= \sum_{t=1}^T \sum_{i=1}^k \E \Big[ \ell^{(i)}_t \big| (s_t, a_t) \Big] - \sum_{t=1}^T \sum_{i=1}^k \E \Big[ \ell^{(i)}_t \big| (s^{\pi^*}_t, a^{\pi^*}_t) \Big].
\end{equation}

We will design efficient online algorithms for the above setting with
low regret. 
}

Since our problem can be formulated as that of learning with a
deterministic MDP with stochastic losses, we could seek to adopt an
existing algorithm for that problem. However, the running-time
complexity of such algorithms would directly depend on the size of the
state space $\sS$, which here is exponential in $k$, and that of the
action set $\sA$. Furthermore, the regret guarantees of these
algorithms would also depend on $|\sS| |\sA|$. Instead, we will show
that, by exploiting the structure of the MDP, we can design vanishing
regret algorithms with a computational complexity that is only
polynomial in $k$ and the number of parameters. More specifically, we
assume access to an oracle that can return the best state, given
the estimated parameters $\bm{\theta}$, that is one that returns the
solution of the optimization \eqref{eq:opt-state}. This optimization
problem is NP-hard even when correlation sets admit a simple
structure.

As an example, consider the case where the correlation sets are
reduced to singletons ($m = 1$), $\sC_i = \set{v_i}$, the loss
distributions corresponding to each criterion are mutually
independent. In that case, given the parameters, finding the optimal
state corresponds to solving a weighted vertex cover problem for which
an approximately optimal solution can be found in polynomial
time. Furthermore, the true parameters of the model can be estimated
accurately by observing at most $k+1$ specific states in $\sS$. See Theorem \ref{thm:opt-state-for-m-equal-1} in Appendix \ref{sec:app-stochastic}.

\paragraph{Case $m=2$.}
%\subsubsection{Case $m = 2$} 
To illustrate the ideas behind our general algorithm, we first consider
a simpler setting where correlation sets are defined on subsets of size
at most two. This setting also captures an important case where fixing a particular criterion affects the rate of fairness complaints of its neighbors. The algorithmic challenge we face here is to avoid exploring the exponentially many states in the MDP. Instead we will design an algorithm that spends an initial exploration phase by visiting a specific subset of at most $4n$. This subset denoted by $\sK$, that we call as the {\em cover} of $\cC$ will help the algorithm estimate the expected loss of any state in the MDP given the estimates of losses for states in the cover. After the exploration phase,
the algorithm creates an estimate $\hat{\bm{\theta}}$ of the true
parameter vector $\bm{\theta}$, uses the optimization oracle for
solving \eqref{eq:opt-state} to find a near optimal state $\hat{s}$
and selects to stay at state $\hat{s}$ for the remaining time steps. 

We next formally define the notion of a cover. For two criteria $i,j$ and $b \in \{0,1\}$, we say that $(i, j, b)$ is a \emph{dichotomy} if there exist two states
$s, s' \in \sS$ such that: (1) $s(j) = 0$ and $s'(j) = 1$, and (2)
$s(i) = s'(i) = b$. We call the two states $s, s'$ an $(i,j,b)$-pair. Note that if an edge $(v_i, v_j)$ is present in $\sG$, then $(i, j, 1)$ cannot be a dichotomy, since criteria $i$ and $j$ cannot be fixed simultaneously. A cover $\sK$ of $\cC$ is simply a subset of the states in the MDP that contains an $(i, j, b)$-pair for every $\{i, j\} \in \cC$ and valid dichotomy $(i,j,b)$. Furthermore, for every singleton set $\{i\}$ in $\cC$, the cover $\sK$ contains states $s, s'$ such that $s(i)=0, s'(i)=1$ and $s(j) = s'(j)$ for all $j \neq i$.
Note that we only need the cover to contain an $(i, j, b)$-pair if $\{i,j\}$ is a correlation set. Hence, it is easy to see that when $m = 2$, there is always a cover of size at most $4n$ in the worst case. 
%However, in many natural cases we expect a cover of size $O(k)$. See Appendix~\ref{sec:app-stochastic}. 
%for more details.

Next, we state our key result showing that, given
the loss values for the states in a cover, we can accurately estimate
the loss values for any vertex in any other state. The proof is in Appendix~\ref{sec:app-stochastic}.
%The proof can be
%found in Appendix~\ref{sec:app-proofs-mdp}.

% \footnote{{\tt Note that there are two different state $s$ in the theorem. Better to change them to be different.}}
\begin{theorem}
\label{thm:phi-cover-m-equal-2}
Let $\sK$ be a cover for $\calC$. 
% and let $s, s' \in \sK$ be an
% $(i, j, b)$-pair. Define $X^{i,j}_b$ as
% $X^{i,j}_b \coloneqq \mu_i^s - \mu_i^{s'}$.
% %\mu_i^{s'} - \mu_i^s.
Then, for any state $s \in \sS$ and any $i \in [k]$ with $s(i) = b$, we have:
\begin{align}
\label{eq:vertex-loss-comp-m-equal-2-main}
\mu_i^s  
& = \mu_i^{s'}  + \sum_{j = 1}^k 
X^{i,j}_{b} \left[ \1(s(j) = 1) \1(s'(j) = 0) - 
\1(s(j) = 0) \1(s'(j) = 1) \right],
\end{align}
where $s'$ is any state in $\sK$ with $s'(i) = b$, and for $\{i,j\} \in \cC$, $X^{i,j}_b \coloneqq \mu_i^{s_1} - \mu_i^{s_2}$ where $(s_1,s_2)$ is some $(i,j,b)$ pair. If $\{i,j\} \notin \cC$, we define $X^{i,j}_b$ to be zero.
\end{theorem}
Based on the above theorem we describe our online algorithm in
Figure~\ref{ALG:stochastic-mdp} and the associated regret guarantee in
Theorem~\ref{thm:regred-mdp-m-equal-2}. The proof can be found in
Appendix~\ref{sec:app-stochastic}.

\begin{figure}[t]
\begin{center}
\fbox{\parbox{1\textwidth}{
{\bf Input:} The graph $\sG$, correlation sets $\calC$, fixing costs $c_i$.
\begin{enumerate}   
\item Pick a cover $\sK =  \set{s_1, s_2, \dots, s_r}$ of $\cC$. 

\item Let $N=10 \frac{T^{2/3}( \log rkT)^{1/3}}{r^{2/3}}$.

\item For each state $s \in \sK$ do:

\begin{itemize}

\item Move from current state to $s$ in at most $k$ time steps.

\item Play action $a=0$ in state $s$ for the next $N$ time steps to
  obtain an estimate $\h \mu_i^s$ for all $i \in [k]$.

\end{itemize}

\item Using the estimated losses for the states in $\sK$ and
  Equation~\eqref{eq:vertex-loss-comp-m-equal-2-main}, 
  %in Appendix~\ref{sec:app-proofs-mdp}, 
  run the oracle for the optimization \eqref{eq:opt-state} to obtain an approximately optimal state $\hat{s}$.
\item Move from current state to $\hat{s}$ and play action $a=0$ from $\hat{s}$ for the remaining time steps.
\end{enumerate}
}}
\end{center}
\caption{Online algorithm for $m = 2$ achieving $\tilde{O}(T^{2/3})$ pseudo-regret.}
\label{ALG:stochastic-mdp} 

\end{figure}

\begin{theorem}
\label{thm:regred-mdp-m-equal-2}
Consider an $\MDP$ with losses in $[0, \B]$, a maximum fixing cost
$\cmax$, and correlations sets of size at most $m = 2$.  Let $\sK$ be
a cover of $\cC$ of size $r \leq 4n$, then, the algorithm of
Figure~\ref{ALG:stochastic-mdp} achieves a pseudo-regret bounded by
$O(k r^{1/3} (\cmax +\B) (\log rkT)^{1/3} T^{2/3})$. Furthermore,
given access to the optimization oracle for \eqref{eq:opt-state}, the
algorithm runs in time polynomial in $k$ and $n = |\cC|$.
\end{theorem}
The algorithm for the case of $m = 2$ naturally extends to arbitrary correlation set sizes via appropriately extending the notion of a dichotomy and a cover with the size of the cover $\sK$ bounded by $n2^{mn}$ in the worst case. See the Algorithm in Figure~\ref{ALG:stochastic-mdp-general} in Appendix \ref{sec:app-stochastic}. This leads to the following general guarantee. 
\begin{theorem}
\label{thm:regred-mdp-m-general-m}
Consider an $\MDP$ with losses bounded in $[0, \B]$ and maximum cost of fixing a vertex being $\cmax$. Given correlations sets $\calC$ of size at most $m$, and a cover $\sK$ of $\calC$ of size $r \leq n 2^{mn}$, the algorithm in Figure~\ref{ALG:stochastic-mdp-general} achieves a pseudo-regret bounded by $O(k r^{1/3} (\cmax +\B) (\log rkT)^{1/3} T^{2/3})$. Furthermore, given access to the optimization oracle for \eqref{eq:opt-state}, the algorithm runs in time polynomial in $k$, $n = |\calC|$ and $r = |\sK|$.
% Given graph $\sG$, arbitrary correlation sets $\calC$, and a cover
% $\sK$ of $\calT$ of size $r \leq n 2^{mt}$, the algorithm in
% Figure~\ref{ALG:stochastic-mdp} achieves a pseudo-regret bounded by
% $O(k (r \log r)^{1/3} (b + c_{\text{max}})^{1/3}
% T^{2/3})$. Furthermore, given access to the optimization oracle for
% \eqref{eq:opt-state} the algorithm runs in time polynomial in $k$,
% $r$, and $n = |\calT|$.
\end{theorem}

\subsection{Beyond $T^{\frac{2}{3}}$ regret.}

%\subsection{Beyond $T^{\frac{2}{3}}$ regret}

In this section, we present algorithms for our 
problem that achieve $\tilde{O}(\sqrt{T})$ regret,
first in the case $m = 1$, next for any $m$, under the natural assumption that each criterion does not
participate in too many correlations sets. Let us first point out that although our problem can be
cast as an instance of the stochastic multi-armed 
bandit problem with switching costs, and arms corresponding to the states in the MDP, existing algorithms \citep{cesa2013online,simchi2019phase} achieving $\tilde{O}(\sqrt{T})$ have time complexity that depends on the number of arms which in our case is exponential ($2^k$). We will show here that, in most realistic instances of our model, we can achieve $\tilde{O}(\sqrt{T})$ regret efficiently.

When correlation sets are of size one, the parameter vector ${\bm \theta}$ can be described using the following $2k$ parameters: for each $i \in [k]$, let $\gamma^0_i$ denote the expected loss incurred by criterion $i$ when it is unfixed and $\gamma^1_i$ its expected loss when it is fixed. Our proposed algorithm is similar to the UCB algorithm for multi-armed bandits \cite{auer2002finite} and maintains optimistic estimates of these parameters. We show that using the estimates of the $2k$ parameters one can construct optimistic estimates for the loss at any given state of the MDP.

For every vertex $i$, we denote by $\tau^0_{i, t}$ the total number of time steps up to $t$ (including $t$) during which the vertex $v_i$ is in an unfixed position and by $\tau^1_{i, t}$ the total number of times steps up to $t$ during which $v_i$ is in a fixed position. Fix $\delta \in (0, 1)$ and let $\hat{\gamma}^b_{i, t}$ be the empirical expected loss observed when vertex $v_i$ is in state $b$, for $b \in \set{0, 1}$. Our algorithm maintains the following optimistic estimates
\begin{align}
    \label{eq:optimistic-estimates-main}
    \tilde{\gamma}^b_{i,t} = \hat{\gamma}^b_{i,t} - 10B\sqrt{\frac{\log (kT/\delta)}{\tau^b_{i,t}}}.
\end{align}

To minimize the fixing cost incurred when transitioning from one state to another, our algorithm works in episodes. In each episode $h$, the algorithm first uses the current optimistic estimates to query the optimization oracle and determine the current best state $s$. Next, it remains at state $s$ for $t(h)$ time steps before querying the oracle again. The number of time steps $t(h)$ will be chosen carefully to avoid incurring the fixing costs too often. 
The algorithm is described in Figure~\ref{ALG:stochastic-mdp-m=1-ucb-main}.
We will prove that it benefits from the following 
regret guarantee.

\begin{figure}[t]
\begin{center}
\fbox{\parbox{1\textwidth}{
{\bf Input:} graph $\sG$, correlation sets $\calC$, fixing costs $c_i$.
\begin{enumerate}   

\item Let $\sK$ be the cover of size $k + 1$ that includes the all zeros state and the states corresponding to indicator vectors of the $k$ vertices. 

\item Move to each state in the cover once and update the optimistic estimates according to \eqref{eq:optimistic-estimates-main}.

\item For episodes $h = 1,2, \dots$ do:

\begin{itemize}

\item Run the optimization oracle \eqref{eq:opt-state} with the optimistic estimates as in \eqref{eq:optimistic-estimates-main} to get a state $s$.

\item Move from current state to state $s$. Stay in state $s$ for $t(h)$ time steps and update the corresponding estimates using \eqref{eq:optimistic-estimates-main}. Here $t(h) = \min_i \tau^{s(i)}_{i,t_h}$ and $t_h$ is the total number of time steps before episode $h$ starts.
\end{itemize}

\end{enumerate}
}}
\end{center}
\caption{Online algorithm for $m = 1$ with $\tilde{O}(\sqrt{T})$ regret.}
\label{ALG:stochastic-mdp-m=1-ucb-main} 
\end{figure}

\begin{theorem}
\label{thm:regred-mdp-m-equal-1-ucb}
Consider an $\MDP$ with losses bounded in $[0, B]$ and maximum cost of
fixing a vertex being $c$. Given correlations sets $\cC$ of
size one, the algorithm of Figure~\ref{ALG:stochastic-mdp-m=1-ucb-main} achieves a
pseudo-regret bounded by
$O(k^2 (c + B)^{2} \sqrt{T} \log T)$. Furthermore,
given access to the optimization oracle for \eqref{eq:opt-state}, the
algorithm runs in time polynomial in $k$.
\end{theorem}
The algorithm of Figure~\ref{ALG:stochastic-mdp-m=1-ucb-main} can be extended to higher $m$ values, assuming that each vertex does not participate in too many correlation sets. The guarantee associated with our general algorithm (see Figure~\ref{ALG:stochastic-mdp-m=higher--ucb} in Appendix \ref{sec:app-stochastic}) leads to the following important corollary.
\begin{corollary}
\label{cor:degree-d-graph}
If $\sG$ is a constant degree graph with correlation sets consisting of subsets of edges in $\sG$, then there is a polynomial time algorithm that achieves a pseudo-regret bounded by $O(k^6 (c + B)^{2} \sqrt{T} \log T)$.
\end{corollary}
%While in the previous sections we study the case when losses are stochastic, 

\section{Adversarial Setting}
\label{sec:adversarial}
In the previous section, we studied a stochastic model for arrival of
complaints and designed no regret algorithms. In this section, we study
the setting when we cannot make assumptions about the arrival of
complaints. In particular, we study an adversarial model where at each time step multiple complaints arrive for the vertices in $\sG$ via the choice made by an oblivious adversary. For a given vertex $v_i$ and time step $t$, we denote by $\ell_{i(t)}$ the loss incurred if criterion $v_i$ is unfixed at time $t$. 
% For a given pair $(i, \ell_{i})$, index $i$
% corresponds to the criterion $v_i$ in the graph $\sG$ to which the
% complaint has been assigned, and $\ell_{i}$ is the non-negative loss
% associated with the complaint. 
Similar to the setting from the
previous section, initially all the vertices in $\sG$ are in unfixed
state and each vertex has a fixing cost of $c_i$. At each time step
the algorithm can decide to fix a particular vertex. As a result all its neighbors get unfixed. At time step $t$, if criterion $v_i$ is unfixed then the the algorithm incurs a loss of $\ell_{i(t)}$. If $v_i$ is fixed at time step $t$ then %criterion
% corresponding to complaint $(i_t, \ell_{i_t})$ is fixed then the
algorithm incurs no loss. 
%Otherwise, the algorithm incurs a loss of
%$\ell_{i_t}$. 
The overall loss incurred by the algorithm is the
total fixing cost and the total loss incurred over the arrival
complaints. As before, we will denote a configuration of the vertices
in $\sG$ using a vector $s \in \{0,1\}^k$ with $s(i)=0$ representing an
unfixed vertex.  For an algorithm $\cA$ processing the request
sequence,
% let $y_{i,t}$ be a $\{0,1\}$ variable such that $y_{i,t}=1$ if
% vertex $i$ is unfixed at time step $t$. For a complaint $x_t$ let
% $i_t$ be the vertex in $\sG$ that the complaint is assigned
% to. Furthermore,
% let $n_i$ be the total number of times vertex $v_i$ is fixed during
% the course of the algorithm. 
During the course of $T$ time steps, the total loss of processing the complaints is
% \begin{align*}
%     \text{Loss}(\cA) = \sum_{t=1}^T \ell_{i(t)} \cdot \1(s_t(i_t) = 0) + \sum_{i=1}^k n_i c_i.
% \end{align*}
\begin{align}
\label{eq:adv-loss-def}
    \text{Loss}(\cA) = \sum_{i=1}^k \sum_{t=1}^T \ell_{i(t)} \cdot \1(s_t(i) = 0) +  \sum_{i=1}^k \sum_{t=2}^T c_i \cdot  \1(s_{t-1}(i)=0, s_t(i)=1).
\end{align}

Define {\OPT} to be the algorithm that given the entire loss
sequence in advance, makes the optimal choice of decisions to fix
vertices. Following standard terminology we define the
\emph{competitive ratio} of an algorithm $\cA$ to be the maximum of
$\text{Loss}(\cA)/\text{Loss}(\OPT)$ over all possible complaint
sequences. We will design efficient online algorithms for processing
the complaints that achieve a constant competitive ratio. Notice that
in this setting, in order for the competitive ratio to be finite, we
need to bound the range of the losses and the fixing costs of the
vertices. We will assume that the cost of fixing each vertex is at
least one and as before assume that the losses are bounded in the range $[0, \B]$. For ease of exposition, in the rest of the discussion we will assume that at each time step complaints arrive for one of the vertices in $\sG$. A simple reduction shows that an algorithm that is competitive with {\OPT} in this setting remains so in the general setting with the same competitive ratio. We discuss this at the end of the section. Via this reduction we can consider the loss sequence to be of the form $((i_1, \ell_{i_1}), \dots, (i_T, \ell_{i_T}))$ where $i_t$ is the index of the criterion for which the $t$th complaint arrives and $\ell_{i_t}$ is the associated loss.
%See Appendix~\ref{sec:app-adversarial} for more details.

To get a better understanding of the above adversarial setting,
consider the case when the graph $\sG$ over the criteria has no edges,
i.e., there are no conflicts. In this case, given a sequence of
complaints, each with unit loss value, the optimal offline algorithm
that has the entire loss sequence in advance can independently make
a decision for each vertex. In particular, if the total loss of the
complaints incurred at vertex $v_i$ exceeds the fixing cost $c_i$ then
the optimal decision is to fix the vertex $v_i$, and otherwise simply
incur the loss from the arriving complaints. In this case the online
algorithm can also simply process each vertex independently. At each
vertex the algorithm is faced with the classical \emph{ski-rental}
problem for which there exists a deterministic algorithm that is
$2$-competitive with optimal algorithm
\cite{karlin1988competitive}. For each vertex $i$, the online
algorithm simply waits till a total loss of $c_i$ or more has been
incurred on vertex $i$ and then decides to fix it. It is easy to see
that the total cost incurred by this strategy is at most twice the
cost incurred by {\OPT}.

However, the above algorithm will fail miserably in the presence of
conflicts in the graph $\sG$. As an example consider a graph with two
vertices $v_i$ and $v_j$ that are connected by an edge. Let the fixing
cost of $v_i$ be $1$ and the fixing cost of $v_j$ be $C \gg
1$. Consider a sequence of complaints, each of unit loss, consisting
of $C$ complaints for $v_j$ followed one complaint for $v_i$. If this
sequence is repeated $T$ times the optimal offline algorithm {\OPT}
incurs a loss of $C+T$ by fixing $v_j$ and incurring losses due to
$v_i$. However, the algorithm above will incur a cost of $(2C+2)T$
thereby leading to an unbounded competitive ratio. Hence, in order to
achieve a good competitive ratio one must make decisions not only
based on the loss incurred at the given vertex $v_i$, but also the
status of the vertices in the neighborhood of $v_i$. Our main result
in this section is the algorithm in Figure~\ref{ALG:adversarial} that
achieves a constant factor competitive ratio.

% \begin{figure}[t]
% \begin{center}
% \fbox{\parbox{1\textwidth}{
% {\bf Input:} The graph $\sG$, fixing costs $c_i$, loss sequence $(i_1, \ell_{i_1}),  \dots, (i_T, \ell_{i_T})$.
% \begin{enumerate}   
% \item For each $i \in [k]$, initialize $\tau_i, \kappa_i$ to $0$.
% \item Process the complaints in sequence and for each complaint $(i, \ell_{i})$ such that $v_i$ is unfixed do:
% \begin{enumerate}
%     %\item $n^{(i)} = n^{(i)} + \ell_{i}$.
%         \item $\tau_i = \tau_i + \ell_{i}$.
%     % \item If $n^{(i)} \geq \max \big( c_i, \sum_{j \in N(i)} m^{(j)}\big)$ fix $v_i$. Set $n^{(i)}$ to $0$ and $m^{(i)}$ to $c_i$. Set $m^{(j)} = 0$ for all $j \in N(i)$. 
%     \item If $\tau_i \geq \max \big( c_i, \sum_{j \in N(i)} \kappa_j\big)$ fix $v_i$. Set $\tau_i$ to $0$ and $\kappa_i$ to $c_i$. Set $\tau_j = 0$ for all $j \in N(i)$. 
%     % \item Otherwise, pick an arbitrary $j \in N(i)$ and set $m^{(j)} = \max(0, m^{(j)} - \ell_{i})$.
%     \item Otherwise, while $\ell_i > 0$ and exists $j \in N(i)$ with $\kappa_j > 0$ do:
%     \begin{enumerate}
%         \item Set $\Delta = \min(\ell_i, \kappa_i)$ and reduce both $\kappa_i$ and $\ell_i$ by $\Delta$.
%     \end{enumerate}
% \end{enumerate}
% \end{enumerate}
% }}
% \end{center}
% \caption{\label{ALG:adversarial} Online algorithm for the adversarial setting.}
% \end{figure}
\begin{figure}[t]
\begin{center}
\fbox{\parbox{1\textwidth}{
{\bf Input:} The graph $\sG$, fixing costs $c_i$, loss sequence $(i_1, \ell_{i_1}),  \dots, (i_T, \ell_{i_T})$.
\begin{enumerate}   
\item For each $i \in [k]$, initialize $\tau_i, \kappa_i$ to $0$.
\item Process the complaints in sequence and for each complaint $(i, \ell_{i})$ such that $v_i$ is unfixed do:
\begin{enumerate}
    %\item $n^{(i)} = n^{(i)} + \ell_{i}$.
        \item $\tau_i = \tau_i + \ell_{i}$.
    % \item If $n^{(i)} \geq \max \big( c_i, \sum_{j \in N(i)} m^{(j)}\big)$ fix $v_i$. Set $n^{(i)}$ to $0$ and $m^{(i)}$ to $c_i$. Set $m^{(j)} = 0$ for all $j \in N(i)$. 
    \item While $\ell_i > 0$ and exists $j \in N(i)$ with $\kappa_j > 0$ do:
    \begin{enumerate}
        \item Set $\Delta = \min(\ell_i, \kappa_i)$ and reduce both $\kappa_i$ and $\ell_i$ by $\Delta$.
    \end{enumerate}    
    \item If $\tau_i \geq \max \big( c_i, \sum_{j \in N(i)} \kappa_j\big)$ fix $v_i$. Set $\tau_i$ to $0$ and $\kappa_i$ to $c_i$. Set $\tau_j = 0$ for all $j \in N(i)$. 
    % \item Otherwise, pick an arbitrary $j \in N(i)$ and set $m^{(j)} = \max(0, m^{(j)} - \ell_{i})$.

\end{enumerate}
\end{enumerate}
}}
\end{center}
\caption{\label{ALG:adversarial} Online algorithm for the adversarial setting.}
\end{figure}

The algorithm described in Figure~\ref{ALG:adversarial} makes
decisions based on local neighborhood information of a
vertex. Intuitively, if a vertex is fixed only once or a few times in
the optimal algorithm one would like to avoid fixing it too many
times. In order to achieve this, each time a vertex $v_i$ is fixed, it
adds a barrier of $\kappa_i = c_i$ to the loss any of its neighbors
need to incur before getting fixed. Hence, if a vertex is connected to a lot of fixed vertices then it has a high barrier to cross before
getting fixed. During the course of the algorithm each unfixed vertex
is in one of the two phases. In phase one, the vertex is accumulating
losses to pay for the barrier introduced by its neighbors (step 2(b) of
the algorithm).
In phase two, once the barrier has been crossed the
vertex follows the standard ski-rental strategy independent of other
vertices for making a decision as to fix or not. Notice that via step 2(b) of the algorithm, multiple neighbors of a vertex $v_i$ can help bring down the barrier of $c_i$ introduced by the action of fixing vertex $v_i$. This is necessary to ensure the online algorithm does not incur a large loss on a vertex by waiting too long to fix it. 

As an example consider a graph $\sG$ with $k$ vertices and $k-1$ edges, where vertex $v_0$ is the central vertex connected to every other vertex. Let the fixing cost of vertex $v_0$ be a large value $C$, and the fixing cost of other vertices be one. We consider a sequence of $C$ complaints, each with unit loss arriving for vertex $v_0$, followed by a sequence of $C$ complaints for vertex $v_1$ and so on. In this case the optimal offline solution incurs a loss of $C+k$ by deciding to fix every vertex except $v_0$. After processing $C$ complaints for $v_0$, the online algorithm will fix $v_0$ and incur a loss of $2C$. Next, during the course of processing $C$ complaints for $v_1$, the algorithm fixes $v_1$ and incurs an additional loss of $C+1$. More importantly, due to step 2(b), the barrier $\kappa_0$ introduced by vertex $v_0$ has been reduced to zero and hence the algorithm only incurs a loss of $2$ per vertex for the remaining sequence for a total loss of $3C+2k-1$. Without the presence of step 2(b) each vertex will incur a loss of $C$ leading to a large competitive ratio.

Notice that our algorithm in Figure~\ref{ALG:adversarial} is designed for a setting where in each time step complaints arrive for a single vertex in $\sG$. If multiple vertices accumulate complaints in a time step, we can simply order them arbitrarily and run the algorithm on the new sequence. Let {\OPT} be the optimal offline algorithm according to the chosen ordering of the complaints. Let {\OPT}' be the optimal offline algorithm when processing multiple complaints per time step. Notice that for each time step, the loss of {\OPT} cannot be larger than that of {\OPT}' since any choice available to {\OPT}' is available to {\OPT} as well. Hence it is enough to design an algorithm that is competitive with {\OPT}. In particular, we have the following theorem. The proof is in Appendix~\ref{sec:adversarial-app}.
% The analysis becomes
% delicate since a vertex may alternate between the two phases numerous
% times. 
% We next state the guarantee associated with our proposed
% algorithm. The proof can be found in
% Appendix~\ref{sec:app-adversarial}.

\begin{theorem}
\label{thm:adversarial-online}
  Let $\sG$ be a graph with fixing costs at least one. Then, the
  algorithm of Figure~\ref{ALG:adversarial} achieves a competitive
  ratio of at most $2\B + 4$ on any sequence of complaints with loss values
  in $[0, \B]$.
\end{theorem}

% \textbf{Adversarial setting.} We also consider an adversarial setting with no distributional assumptions on the loss, for which we give efficient near-optimal algorithms, see Appendix~\ref{sec:adversarial}.

%\section{Experiments}
%\label{sec:experiments-main}
\section{Experiments}
\label{sec:experiments}

In this section we evaluate the performance of our algorithms developed in the stochastic setting of Section~\ref{sec:stochastic}. We consider a simulated environment where the conflict graph $\cG$ is generated from the Erd\H{o}s-Renyi model: $G(k,p)$ where we set $p = 2\frac{\log k}{k}$. This ensures that with high probability $\cG$ is connected. 
% \footnote{{\tt YM: I assume this is to make sure the graph is connected, but very sparse. right}} 
Next we generate correlation sets $\cC$ consisting of pairs of vertices in $\cG$ sampled uniformly at random. For a parameter $\alpha > 0$ that we vary, we choose $\alpha k$ pairs of vertices at random and add them as correlation sets in $\cC$. Hence on average, each vertex participates in $\alpha$ correlation sets.
% We consider two regimes: {\em sparse} regime where we sample $0.2 \cdot k$ pairs and {\em dense} regime where we sample $0.2 \cdot \binom{k}{2}$ pairs of vertices. 
We also add to $\cC$ singleton sets for each vertex in $\cG$. The fixing cost of each vertex is samples uniformly at random in the range $[1,5]$. 

Next we describe the choice of parameters governing the loss distribution of the different states in the MDP. For a correlation set $\{i\}$ of size one corresponding to vertex $v_i$, we sample a parameter $\gamma^1_i$ from the beta distribution $\text{Beta}(0.5, 0.5)$. For a given state $s$ with $s(i)=1$, the loss generated due to $\{i\}$ is drawn from an exponential distribution with mean $\gamma^1_i$. For a given state $s$ with $s(i)=0$, the loss generated due to $\{i\}$ is drawn from an exponential distribution with mean $\lambda \gamma^1_i$, where $\lambda > 1$ is a parameter that we vary. For a correlation set $\{i,j\}$ of size two, we generate two parameters $\gamma^{1,1}_{i,j}$ and $\gamma^{1,0}_{i,j}$ from the beta distribution $\text{Beta}(0.5, 0.5)$ such that $\gamma^{1,0}_{i,j} > \gamma^{1,1}_{i,j}$. For a given state $s$ with $s(i)=1$ and $s(j)=1$, the loss generated due to $\{i,j\}$ is drawn from an exponential distribution with mean $\gamma^{1,1}_{i,j}$. For states where $s(i)=0$ and $s(j)=1$ or vice-versa, the loss is generated from an exponential distribution with mean $\gamma^{1,0}_{i,j}$. Finally, for states where both $s(i)=0$ and $s(j)=0$, the loss is generated from an exponential distribution with mean $\lambda \gamma^{1,0}_{i,j}$. 
% The above ensures that the loss distribution is {\em monotonic}, namely that unfixing a vertex from a given state cannot decrease the expected loss incurred for any vertex. 

In general, computation of the optimal state in \eqref{eq:opt-state} requires time exponential in $k$. In our experiments we approximate the optimal state by a linear programming relaxation of the optimization in \eqref{eq:opt-state} and use the appropriately rounded linear programming relaxation solution as a proxy for the optimal state.  
\begin{figure}[h]
\centering
\subfloat{\includegraphics[width=4.5cm]{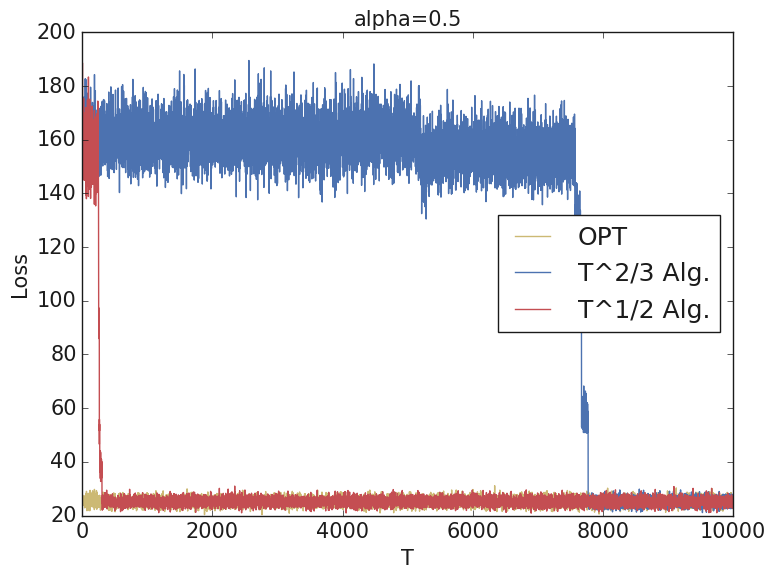}}\hfil
\subfloat{\includegraphics[width=4.5cm]{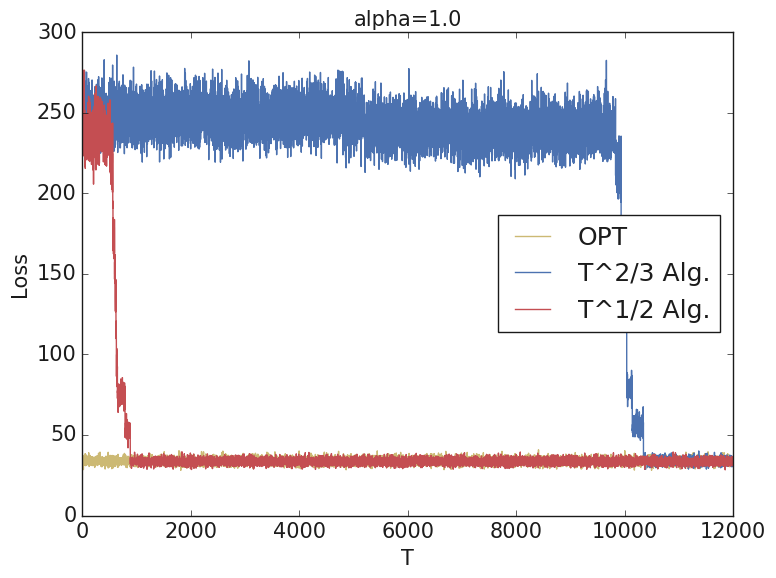}}\hfil 
\subfloat{\includegraphics[width=4.5cm]{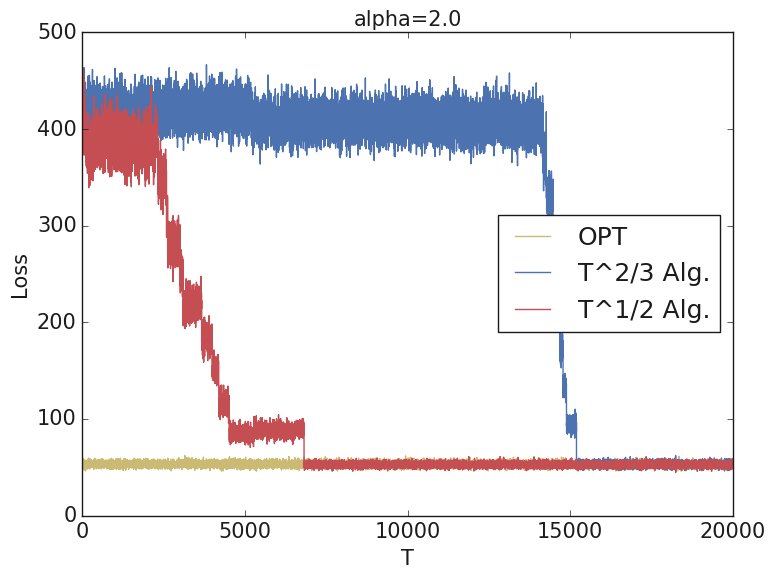}} 

\subfloat{\includegraphics[width=4.5cm]{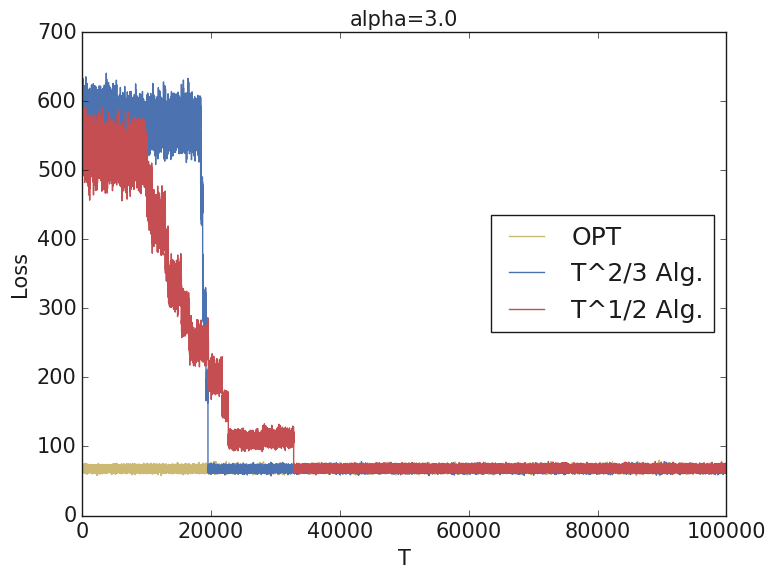}}\hfil   
\subfloat{\includegraphics[width=4.5cm]{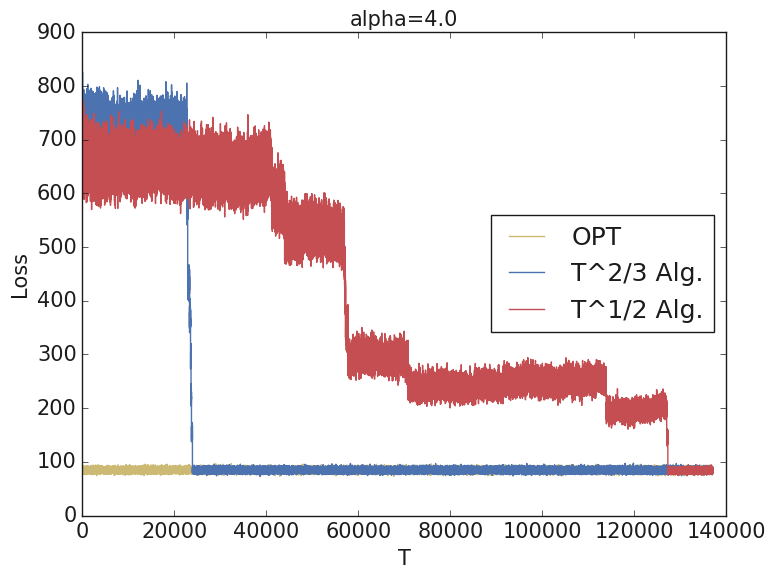}}\hfil
\subfloat{\includegraphics[width=4.5cm]{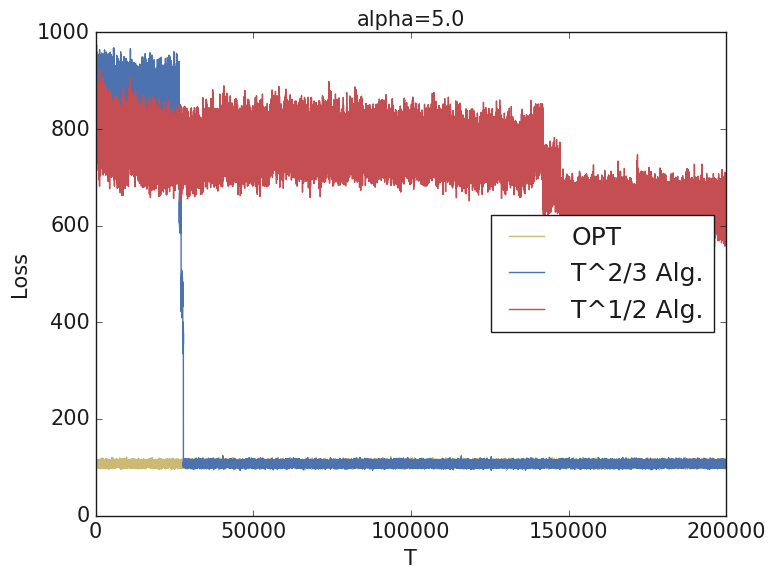}}
\caption{The figure shows the total accumulated loss incurred by the Algorithms in Figure~\ref{ALG:stochastic-mdp} and Figure~\ref{ALG:stochastic-mdp-m=higher--ucb} on a graph with $k=50$ criteria. The parameter $\alpha$ controls the total number of correlation sets. For each value of $\alpha$, we add $\alpha k$ random pairs of vertices into correlation sets. \label{fig:reg_vs_alpha}}
\end{figure}

\begin{figure}[h]
\centering
\subfloat{\includegraphics[width=4.5cm]{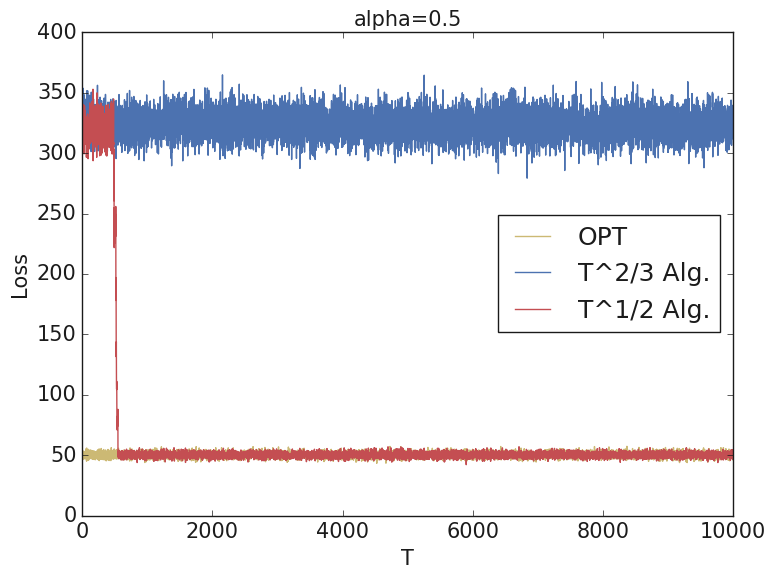}}\hfil
\subfloat{\includegraphics[width=4.5cm]{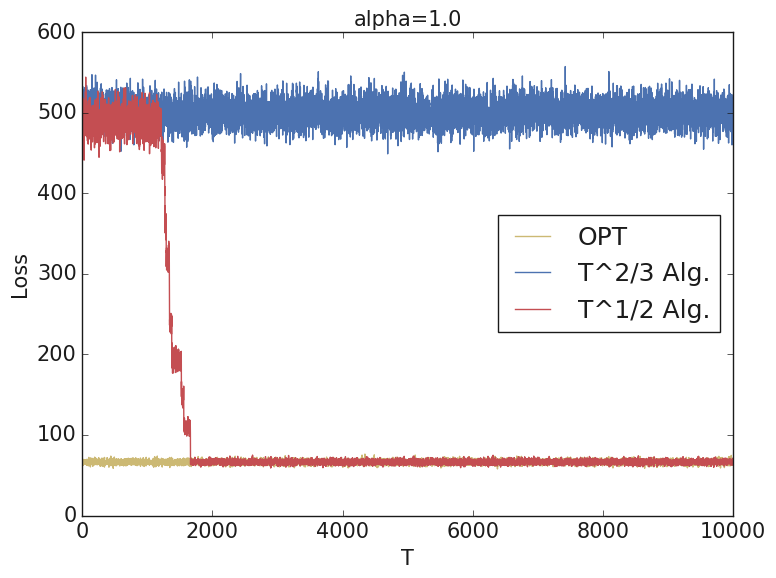}}\hfil 
\subfloat{\includegraphics[width=4.5cm]{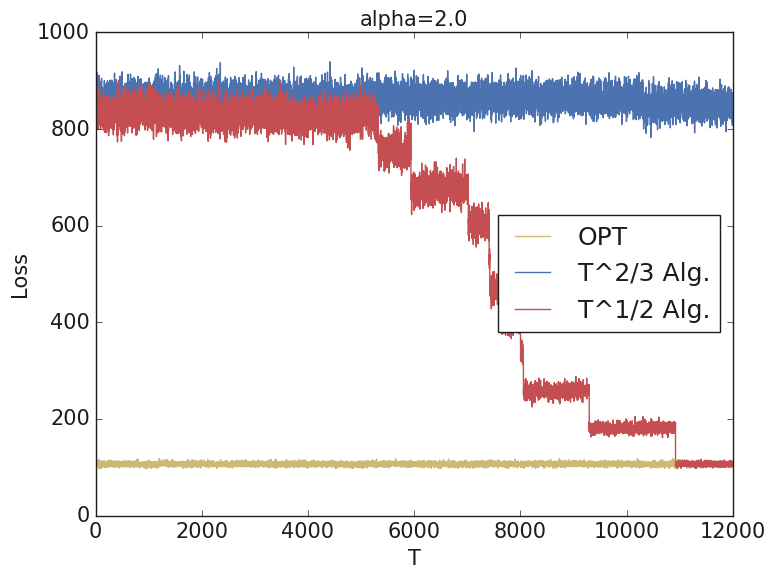}} 

\subfloat{\includegraphics[width=4.5cm]{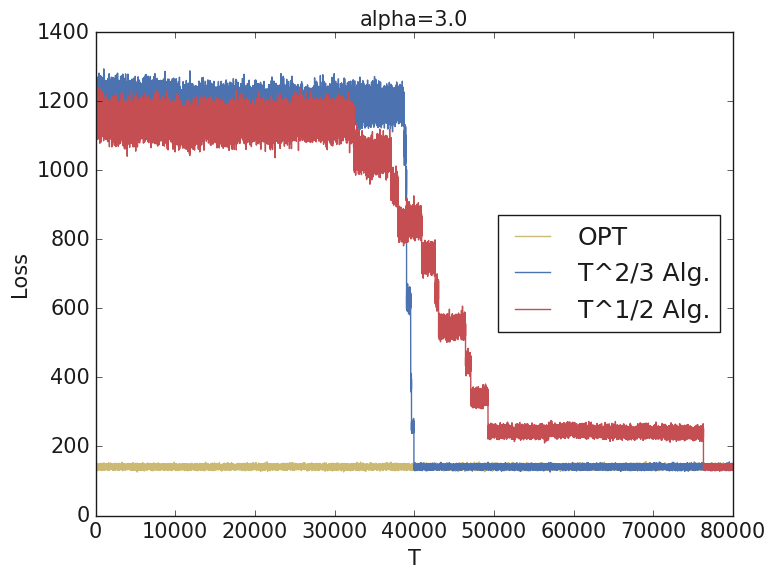}}\hfil   
\subfloat{\includegraphics[width=4.5cm]{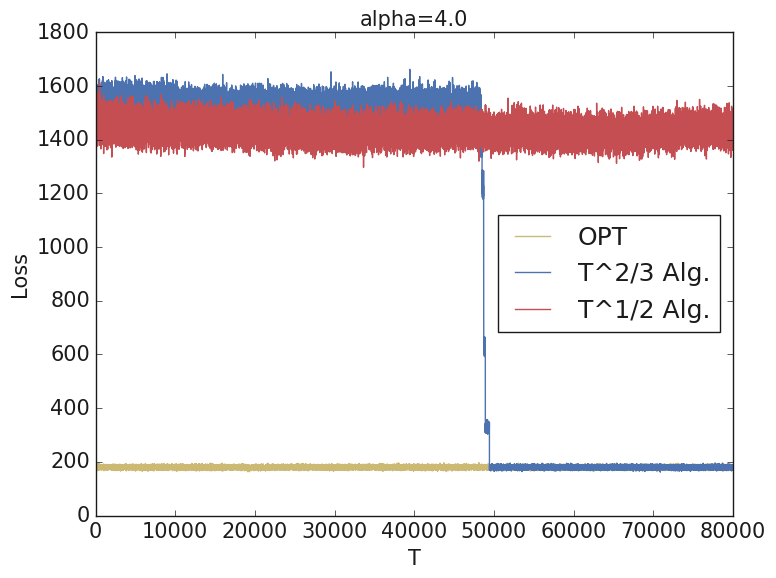}}\hfil
\subfloat{\includegraphics[width=4.5cm]{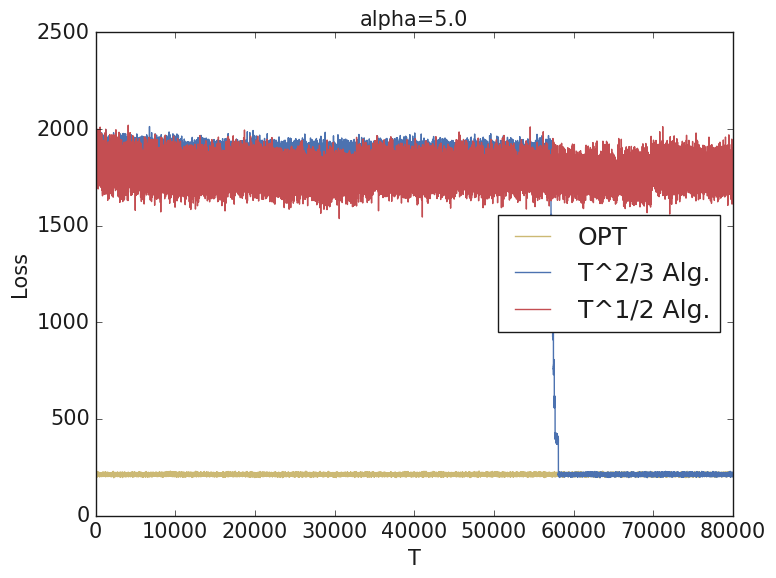}}
\caption{The figure shows the total accumulated loss incurred by the Algorithms in Figure~\ref{ALG:stochastic-mdp} and Figure~\ref{ALG:stochastic-mdp-m=higher--ucb} on a graph with $k=100$ criteria. The parameter $\alpha$ controls the total number of correlation sets. For each value of $\alpha$, we add $\alpha k$ random pairs of vertices into correlation sets. \label{fig:reg_vs_alpha_k_100}}
\end{figure}

For general $m$, our proposed algorithms in Figure~\ref{ALG:stochastic-mdp} and Figure~\ref{ALG:stochastic-mdp-m=higher--ucb} have complementary strengths. While the algorithm in Figure~\ref{ALG:stochastic-mdp} incurs a higher regret as a function of the number of time steps $T$, its running time has a polynomial dependence on the parameter $\alpha$, i.e., the number of correlation sets that a vertex participates in, on average. The algorithm in Figure~\ref{ALG:stochastic-mdp-m=higher--ucb} incurs a smaller regret of $\tilde{O}(\sqrt{T})$ as a function of $T$ at the expense of an exponential dependence on $\alpha$. In Figures \ref{fig:reg_vs_alpha} and \ref{fig:reg_vs_alpha_k_100} we empirically demonstrate this behavior where for small values of $\alpha$, the $\tilde{O}(\sqrt{T})$-regret algorithm is much better, whereas for higher values of $\alpha$ the $\tilde{O}(T^{2/3})$-regret algorithm is more desirable.

For the case of $m=1$ however, i.e., singleton correlation sets, the algorithm in Figure~\ref{ALG:stochastic-mdp-m=higher--ucb} achieves a smaller regret and runs in polynomial time and hence is expected to outperform the explore-exploit based algorithm from Figure~\ref{ALG:stochastic-mdp}. As can be seen from Figure~\ref{fig:reg_m_1} this is indeed the case and the $\tilde{O}(\sqrt{T})$ regret algorithm significantly outperforms the $\tilde{O}(T^{2/3})$ regret algorithm.

\begin{figure}[t]
\centering
\subfloat{\includegraphics[width=4.5cm]{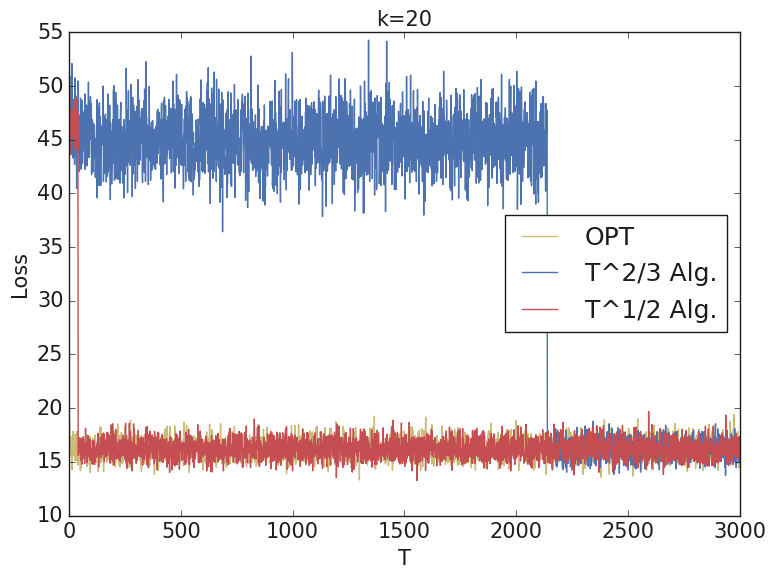}}\hfil
\subfloat{\includegraphics[width=4.5cm]{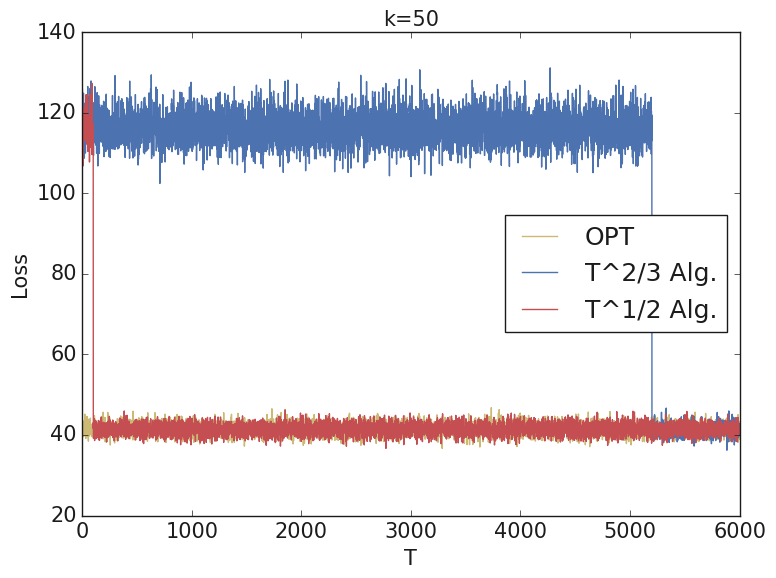}}

\subfloat{\includegraphics[width=4.5cm]{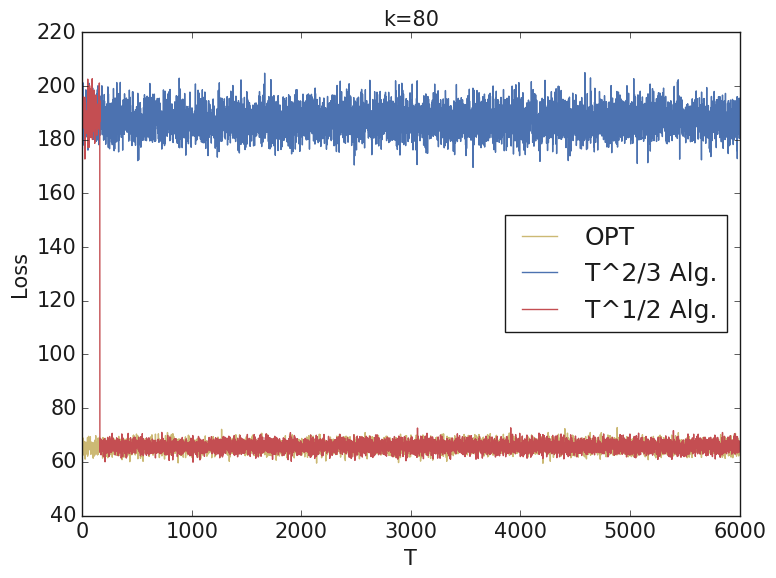}}\hfil   
\subfloat{\includegraphics[width=4.5cm]{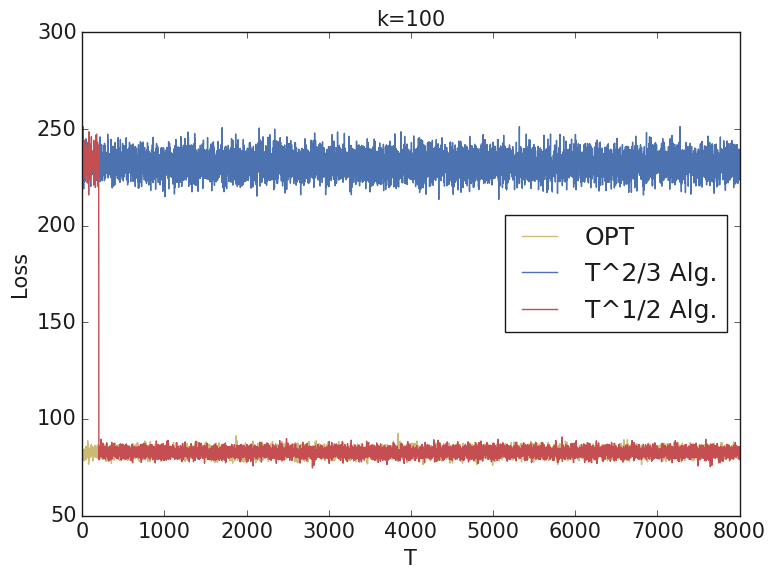}}\caption{The figure shows the total accumulated loss incurred by the Algorithms in Figure~\ref{ALG:stochastic-mdp} and Figure~\ref{ALG:stochastic-mdp-m=1-ucb-main} for the case of $m=1$ and varying graph sizes. \label{fig:reg_m_1}}
\end{figure}

\section{Related Work}
\label{sec:app-related}

Much of the current research on algorithmic fairness has focused on
designing solutions tailored to specific criteria such as
\emph{equalized odds} \citep{hardt2016equality}, \emph{counterfactual}
or \emph{demographic parity} \citep{KusnerLoftusRussellSilva17},
\emph{calibration} \citep{guo2017calibration}, or criteria for
individual fairness \citep{zemel2013}.

There has been a series of recent work on optimizing multiple fairness
constraints via constrained non-convex optimization. These
publications either reduce the problem to that of cost-sensitive
classification \citep{agarwal2018reductions, dwork2018decoupled} or
replace the non-convex constraints by convex proxies and next
optimize them via external or swap regret minimization algorithms
\citep{cotter2018two, cotter2018training}.

Classification has been the main focus of the research on algorithmic fairness.  There has also been work, however, studying fairness
criteria and algorithms for ranking \citep{celis_fair_ranking,
  beutel2019fairness, narasimhan2019pairwise} and clustering problems
\citep{fair_clustering_Nips2017, schmidt2018fair, backurs2019scalable}
with requirement of a proportional representation of populations
within each cluster. There have been also studies of the problem of fair
selection \citep{KearnsRothWu2017,KleinbergRaghavan2018}, online
learning and multi-armed
bandit problems \citep{LiuRadanovicDimitrakakisMandalParkes2017,
  JosephKearnsMorgensternNeelRoth2018}, 
in particular stochastic and
contextual bandits 
\citep{joseph2016fairness, gillen2018online,
  gupta2019individual} and
reinforcement learning
\citep{jabbari2017fairness,doroudi2017importance, wen2019fairness,
  ZhangShah2014}, human-classifier hybrid decision systems
\citep{MadrasPitassiZemel2018}, or fair personalization
\citep{CelisVishnoi2017}.

There have also been studies of the inherent tension between
satisfying multiple fairness metrics and composition of fair
classifiers. \cite{kleinberg2017} and \cite{feller2016computer}
demonstrate that it is impossible to satisfy equal opportunity and
calibration at the same time. \cite{blum2018preserving,
  dwork2018decoupled, dwork2020individual} study whether a composition
of fair classifiers remains fair and provided negative
results. \cite{menon2018cost} study the accuracy vs. fairness tradeoff
in classification problems with group fairness constraints.

\cite{joseph2016fairness} study a notion of individual
fairness in the context of stochastic bandits. There are $k$ arms and
in each time step, the algorithm picks an arm according to
distribution $p_t$. For arm $i_t$ a stochastic reward is revealed with
mean vector $\mu_i$. Furthermore, the authors impose the fairness
constraint that at each time step $t$, the algorithm can pick an arm
$i$ with higher probability than arm $j$ only if
$\mu_i > \mu_j$ (with high probability).

\cite{gillen2018online} extend the stochastic bandits setting to the
case of linear contextual bandits. Here, $k$ adversarially chosen
context vectors $x^t_1, \dots, x^t_k$ arrive at time $t$ and the mean
reward of playing arm $i$ at time $t$ is $\theta \cdot x^t_i$ for an
unknown vector $\theta$. The notion of individual fairness considered
is that if two contexts $x^t_i$ and $x^t_j$ are close to each other,
then the corresponding arms should be played with similar
probabilities. More specifically, it must holds that
$|p_t(i) - p_t(j)| \leq d(x^t_i, x^t_j)$. The distance metric $d$ is
assumed to be a Mahalanobis distance with an unknown matrix $A$ and
after each round it is revealed which individual fairness criteria
were violated.

While we avoid using special structures of specific fairness metrics
in our algorithm design, several of the studies mentioned before can be
cast into our framework. There are, however, several important
differences that make the algorithmic guarantees incomparable as these
studies aim to design algorithms tailored to a specific metric. For
instance, to model the stochastic bandit setting of
\cite{joseph2016fairness} as described above, we could consider a
graph $\sG$ with $\binom{k}{2}$ vertices, one for each pair of
arms. Each vertex captures the individual fairness metric associated
with the pair. More importantly, the graph will have no edges. Hence,
the states in the MDP can be described by a vector in
$\set{0, 1}^{\binom{k}{2}}$ describing which pairs of arms violate
individual fairness. The actions will consist of picking a probability
vector $p_t$ over the $k$ arms in each time step. However, in the
setting of \cite{joseph2016fairness}, the transition dynamics are
unknown since fairness violations as a result of taking an action
$p_t$ depend on the true unknown means of the arms. Since the work of
\cite{joseph2016fairness} studies a particular fairness metric, one
can use special structure of the problem to design a customized
solution. In our model, we study arbitrary fairness metrics and thus,
some assumptions about the transition dynamics are necessary.

\ignore{ 

  Again in our model the above setting would correspond to
  ${k \choose 2}$ vertices in the graph with no conflicting
  edges.\footnote{{\tt YM: we have $k^2$ vertices per time step, since
      the $x_i^t$ change over time. }} The state transitions are now
  governed by both the probability vector $p_t$ and the context
  vectors $x^t_1, \dots, x^t_k$. In this case the oracle to compute a
  probability vector $p_t$ that takes us to a particular state would
  correspond to a constrained loss minimization with constraints of
  the type: $|p_t(i) - p_t(j)| \leq d(x^t_i, x^t_j)$. The main
  difference is that in the setting of \cite{gillen2018online} the
  loss of each state is fixed. If $r$ criteria are violated the
  fairness loss is exactly $r$. Hence the goal is to navigate the
  states to quickly learn the unknown distance matrix $A$ and minimize
  the fairness loss. We on the other hand navigate the state space to
  learn the loss distribution.\footnote{{\tt YM: I am confused what it
      says that we are doing}}

}
    
Recent works have also studied the long-term impact of optimizing
fairness criteria in settings with feedback mechanisms
\citep{liu2018delayed, hashimoto2018fairness, mouzannar2019fair,
  kannan2019downstream}. \cite{liu2018delayed} show that,
in certain situations, constrained loss minimization to
equalize certain fairness criteria could lead to further disparate
impact in the long run. \cite{hashimoto2018fairness}
proposed algorithms for minimizing such disparate impact in settings
involving repeated loss minimization. More recently,
\cite{jabbari2017fairness, wen2019fairness} study the problem of
satisfying fairness constraints in reinforcement learning settings
involving a Markov Decision Process. The authors in
\cite{jabbari2017fairness} consider learning in an MDP where the
fairness criteria requires that the algorithm never takes an action $a$
over action $a'$ if the long-term reward is higher. It is clear to see
that the optimal policy for the MDP indeed satisfies this
property. Hence, there does exist a fair policy. However, the authors
show that finding a near optimal fair policy requires time exponential
in the size of the state space.

\cite{wen2019fairness} consider other fairness metrics such as
demographic parity in the context of learning in MDPs.
\cite{doroudi2017importance} show that existing importance sampling
methods for off-policy policy selection in reinforcement learning can
lead to unfair outcomes and present algorithms to mitigate this
effect. \cite{ZhangShah2014} define a fairness solution criterion for
multi-agent MDPs. The recent work of \citet{mladenov2020optimizing} studies optimzing long term social welfare in recommender systems.

While our work also involves learning in a Markov Decision Process
(MDP) and optimizing fairness in the long term, the setup and the motivation are different. Unlike all the
previous work mentioned, we do not commit to a fixed definition of
fairness and allow for arbitrary fairness criteria. Hence, states in
our MDP correspond to the current configurations of different fairness
criteria. Rather than studying each fairness metric in isolation, the
objective of our work is to propose a data-driven model that can learn
from feedback, a near-optimal configuration of the metrics to impose
on the system. To the best of our knowledge, ours is the first work to
incorporate optimizing fairness metrics of arbitrary types in an
online setting. In this context, the recent work of
\cite{kearns2019average} studies a specific combination of group and
individual fairness metrics. The authors consider a setting where
there is a distribution over individuals as well as a distribution
over classification tasks. They consider algorithms for achieving
\emph{average} individual fairness, that is in expectation over
classification tasks, the performance of the algorithm on a group
fairness metric such as demographic parity should be the same for each
individual.

An important aspect of our stochastic MDP-based model requires the
ability to observe the losses associated with different fairness
criteria at each time.  This relates to the problem of evaluating
(un)fairness according to different metrics from data.  Many such
metrics require access to both labeled data and to sensitive
attribute information such as race or gender, for accurate
evaluation. A recent line of work has studied this estimation problem
when one has limited and/or noisy access to sensitive attribute
information \citep{gupta2018proxy, coston2019fair, lamy2019noise,  wang2020robust}.

\section{Conclusion}

We presented a new data-driven model of fairness in the presence of
multiple criteria, with algorithms benefitting from theoretical
guarantees both in the stochastic and in the adversarial setting.
While we believe that our model can be realized in practice and while
our experiments show the effectiveness of our algorithms empirically,
several extensions are worth exploring to further expand their
applicability. These include fixing costs that can vary with time to
capture the varying algorithmic price and human effort required to
address various fairness criteria. Similarly, the expected losses in
our stochastic model could be time-dependent to express the growing
cost of a fairness criterion not being addressed.

\bibliographystyle{abbrvnat}
\bibliography{fair}

\newpage
\appendix

\section{Stochastic Setting}
\label{sec:app-stochastic}

We first show that in the stochastic model, if correlation sets are of size one then one can efficiently approximate the cost of the optimal state up to a factor of two.

\begin{theorem}
\label{thm:opt-state-for-m-equal-1}
If correlations sets are of size one ($m = 1$), then, for any
$\e, \delta > 0$, the true parameter vector for $\MDP$ can be
approximated to $\e$-accuracy in $\ell_\infty$-norm with probability
at least $1 - \delta$, in at most
$O(\frac{\B^2 k}{\epsilon^2} \log (\frac{k}{\delta}))$ time steps and
exploring at most $k+1$ specific states in $\sS$. Furthermore, given a
parameter vector $\bm{\theta}$, there is an algorithm that runs in
time polynomial in $k$ and finds an approximately optimal state $s'$
such that $g(s') \leq 2 \min_{s \in \sS} g(s)$.
\end{theorem}
\begin{proof}
Notice that when correlation sets are of size one, the expected loss incurred for criterion $v_i$ at any given state $s$ solely depends on whether $s(i)=0$ or $s(i)=1$. Hence in this case the MDP consists of $2k$ parameters where we use $\gamma^1_{i}$ and $\gamma^0_i$ to denote the expected losses incurred by vertex $i$ when it is in fixed and unfixed position respectively. For any $\delta > 0$, by Hoeffding's inequality, we have that if we stay in state $s = (0,0, \dots, 0)$ for $N = \frac{B^2}{\epsilon^2} \log(2k/\delta)$ time steps then with probability at least $1-\frac{\delta}{2}$, we have each $\gamma^0_i$ estimated up to $\epsilon$ accuracy. Let $e_i \in \{0,1\}^k$ denote the indicator vector for $i$. If we stay in state $s = e_i$ for $\frac{B^2}{\epsilon^2} \log(2k/\delta)$ time steps, then with probability at least $1-\frac{\delta}{2}$ we have $\gamma^1_i$ estimated up to $\epsilon$ accuracy. Hence, overall after $O(\frac{B^2 k}{\epsilon^2} \log (\frac{k}{\delta}))$ time steps, we have each parameter estimated up to $\epsilon$ accuracy. Notice that in total we observe at most $k+1$ states.

Next we show how to efficiently approximate the loss of the best state. Given the parameters of the MDP each vertex has two costs $\Lambda^{(1)}_i = \gamma^0_i$, denoting the cost incurred if the vertex is unfixed and
$\Lambda^{(2)}_i = c_i + \gamma^1_i$, denoting the cost incurred if the vertex is fixed. Without loss of generality assume that $\Lambda^{(1)}_i > \Lambda^{(2)}_i$~(any
vertex that does not satisfy this can be safely left unfixed). For each $i$, define $y_i=1$ if vertex $i$ is
unfixed otherwise define $y_i=0$. Then the offline problem of finding the best state can be written as 
\begin{align*}
    \text{min} & \sum_{i=1}^k (1-y_i) \Lambda^2_i + y_i\Lambda^1_i = \sum_{i=1}^k y_i \gamma_i + \sum_{i=1}^k \Lambda^{(2)}_i\\
    \text{s.t. } & y_i \in \{0,1\}\\
    y_i + y_j &\geq 1, \,\, \forall (v_i,v_j) \in E.
\end{align*}
Here $\gamma_i = \Lambda^{(1)}_i - \Lambda^{(2)}_i > 0$. By relaxing $y_i$ to be in
$[0,1]$ and solving the corresponding linear programming relaxation, we get a solution $y^*_1,
y^*_2, \dots, y^*_k$. Let $\text{LPval}$ denote the linear programming objective value achieved by $y^*_1,
y^*_2, \dots, y^*_k$. Since the linear programming formulation is a valid relaxation of the problem of finding the best state, we have $\text{LPval} \leq \min_{s \in \sS} g(s)$.

We output the state $s'$ in which a vertex $i$ if and only if $y^*_i <
1/2$. Let $S$ be the set of fixed vertices. We have
\begin{align*}
    g(s') &= \sum_{i \in S} \Lambda^{(2)}_i + \sum_{i \notin S} \Lambda^{(1)}_i\\
    &= \sum_{i=1}^k \Lambda^{(2)}_i + \sum_{i \notin S} (\Lambda^{(1)}_i - \Lambda^{(2)}_i)\\
    &= \sum_{i=1}^k \Lambda^{(2)}_i + \sum_{i \notin S} \gamma_i\\
    &< \sum_{i=1}^k \Lambda^{(2)}_i + 2\sum_{i \notin S} y^*_i\gamma_i\\
    &< 2 \Big(\sum_{i=1}^k \Lambda^{(2)}_i + \sum_{i=1}^k y^*_i\gamma_i \Big)\\
    &< 2 \cdot \text{LPval}\\
    &\leq \min_{s \in \sS} 2 g_{\bm{p}}(s).
\end{align*}
\end{proof}
% See Appendix~\ref{sec:app-proofs-mdp} for the proof of the theorem.

\subsection{Case $m = 2$} 

To illustrate the ideas behind our general algorithm, we first consider
a simpler setting where correlation sets are defined on subsets of size
at most two. This setting also captures an important case where fixing
a particular criterion affects the rate of fairness complaints of its
neighbors.

Our algorithm consists of an exploration phase where it observes the
losses for a specific subset of at most $4n$ states. We will show that after
the exploration phase, the algorithm can accurately estimate the
expected loss for any other state $s \in \sS$. Notice that the number
of states in $\sS$ is in general exponential in $k$. Thus, the subset
of states to observe must be carefully chosen and must take into
account the structure of the graph $\sG$. After the exploration phase,
the algorithm creates an estimate $\hat{\bm{\theta}}$ of the true
parameter vector $\bm{\theta}$, uses the optimization oracle for
solving \eqref{eq:opt-state} to find a near optimal state $\hat{s}$
and selects to stay at state $\hat{s}$ for the remaining time steps.

Let $\cmax$ denote the maximum fixing cost:
$\cmax = \max_{i \in [k]} c_i$.  We will show that the pseudo-regret of our
algorithm is bounded by $O(k \log k (\cmax +\B)^{1/3} T^{2/3} \log (kT))$. We
first describe how we select the subset of states to observe in the
exploration phase.

We say that $(i, j, b)$ is a \emph{dichotomy} if for two criteria $i$
and $j$ and for $b \in \set{0, 1}$, there exist two states
$s, s' \in \sS$ such that: (1) $s(j) = 0$ and $s'(j) = 1$, and (2)
$s(i) = s'(i) = b$. Note that if an edge $(v_i, v_j)$ is present in
$\sG$, then $(i, j, 1)$ cannot be a dichotomy, since criteria $i$ and
$j$ cannot be fixed simultaneously.

\begin{definition}
\label{def:cover}
Consider a subset $\sK \subset \sS$. We will say that $\sK$ is a
\emph{cover} for $\cC$ if for any dichotomy $(i, j, b)$, where
$\set{i, j}$ is a correlation set ($\set{i, j} \in \cC$) there exist
two states $s, s' \in \sK$ such that:
\begin{enumerate}

\item[(1)] they agree in all criteria except criterion $j$: $s(l) = s'(l)$
for all $l \neq j$;

\item[(2)] criteria $i$ is in state $b$ in both: $s(i) = s'(i) = b$;

\item[(3)] we have that $s(j) = 0$ and $s'(j) = 1$. 

\end{enumerate}
We call such a pair $(s, s')$ an $(i, j, b)$-pair.
\end{definition}
Furthermore, for every singleton set $\{i\}$ in $\cC$, the cover $\sK$ contains states $s,s'$ such that $s(i)=0, s'(i)=1$ and $s(j) = s'(j)$ for all $j \neq i$. We can always find a cover $\sK$ of size at most $4n$ by picking for
each $\set{i, j} \in \cC$, at most four states corresponding to
different bit configurations for $i$ and $j$, with all other bits set
to zero. 
% However, for many natural graph families, we
% expect the size of the cover to be $O(k)$. As an example consider the case when the graph $\sG$ has no edges and the correlation sets in $\cC$ consist of all pairs of vertices. In this case there exists a specific cover $\sK$ of size $k+1$, that includes the all zero state in which each vertex is unfixed, and one state per vertex $v_i$ where only $v_i$ is fixed. 
For any valid dichotomy $(i,j,b)$ we define $X^{i,j}_b$ as
\begin{align}
    \label{eq:def-X-m=2}
    X^{i,j}_b \coloneqq \mu_i^s - \mu_i^{s'},
\end{align}
where $s,s' \in \sK$ is an $(i,j,b)$ pair. If $\{i,j\} \notin \cC$ we define $X^{i,j}_b$ to be zero. Notice that the values $X^{i,j}_b$ can be approximated from estimating the loss values of states in the cover. Next, we state our key result showing that, given
the loss values for the states in a cover, we can accurately estimate
the loss values for any vertex in any other state. 
% The proof can be
% found in Appendix~\ref{sec:app-proofs-mdp}.

% \footnote{{\tt Note that there are two different state $s$ in the theorem. Better to change them to be different.}}
\begin{reptheorem}{thm:phi-cover-m-equal-2}
Let $\sK$ be a cover for $\calC$. 
% and let $s, s' \in \sK$ be an
% $(i, j, b)$-pair. Define $X^{i,j}_b$ as
% $X^{i,j}_b \coloneqq \mu_i^s - \mu_i^{s'}$.
% %\mu_i^{s'} - \mu_i^s.
Then, for any state $s \in \sS$ and any $i \in [k]$ with $s(i) = b$, we have:
\begin{align}
\label{eq:vertex-loss-comp-m-equal-2-main-app}
\mu_i^s  
& = \mu_i^{s''}  + \sum_{j = 1}^k 
X^{i,j}_{b} \left[ \1(s(j) = 1) \1(s''(j) = 0) - 
\1(s(j) = 0) \1(s''(j) = 1) \right],
\end{align}
where $s''$ is any state in $\sK$ with $s''(i) = b$.
\end{reptheorem}
\begin{proof}
Consider a correlation set $\{i,j\}$. The expected loss incurred by vertex $v_i$ or $v_j$ due to this set in any given state depends solely on the configuration of $v_i$ and $v_j$ in that state. Hence there are four parameters in the $\bm{\theta}$ vector corresponding to the correlation set $\{i,j\}$ and we denote them using $\gamma^{a,b}_{i,j}$, where $a,b \in \{0,1\}$.
Let $s, s'$ be an $(i,j,b)$ pair. When we switch from $s$ to $s'$ the only difference in the expected losses for vertex $i$ comes from the pair $(i,j)$. Hence we have
\begin{align*}
\mu^{s'}_i - \mu^s_i &= \gamma^{b,1}_{i,j} - \gamma^{b,0}_{i,j}
\coloneqq X^{i,j}_b.
\end{align*}
Hence, given the loss estimates for states in $\sK$ we can estimate $X^{i,j}_b$ for each $i,j \in [k]$ and $b \in \{0,1\}$. Next, given an arbitrary state $s$ with $s(i)=b$ let $s'' \in \sK$ such that $s''(i)=b$. We have
\begin{align*}
    \mu^s_i &= \mu^{s''}_i + \sum_{\substack{j: s(j)=0 \\ s''(j)=1}} (\gamma^{b,0}_{i,j} - \gamma^{b,1}_{i,j}) + \sum_{\substack{j: s(j)=1 \\ s''(j)=0}} (\gamma^{b,1}_{i,j} - \gamma^{b,0}_{i,j})\\
     &= \mu^{s''}_i + \sum_{\substack{j:s(j)=1, \\ s''(j)=0}} X^{i,j}_{b}-\sum_{\substack{j:s(j)=0, \\ s''(j)=1}} X^{i,j}_{b}\\
    &= \mu_i^{s''}  + \sum_{j = 1}^k 
X^{i,j}_{b} \left[ \1(s(j) = 1) \1(s''(j) = 0) - 
\1(s(j) = 0) \1(s''(j) = 1) \right].
\end{align*}
\end{proof}
Based on the above theorem we describe our online algorithm in
Figure~\ref{ALG:stochastic-mdp-app} (same as the algorithm in Figure~\ref{ALG:stochastic-mdp} of the main body) and the associated regret guarantee. 
% The proof can be found in
% Appendix~\ref{sec:app-proofs-mdp}.

\begin{figure}[t]
\begin{center}
\fbox{\parbox{1\textwidth}{
{\bf Input:} The graph $\sG$, correlation sets $\calC$, fixing costs $c_i$.
\begin{enumerate}   
\item Pick a cover $\sK =  \set{s_1, s_2, \dots, s_r}$ of $\cC$. 

\item Let $N=10 \frac{T^{2/3}( \log rkT)^{1/3}}{r^{2/3}}$.

\item For each state $s \in \sK$ do:

\begin{itemize}

\item Move from current state to $s$ in at most $k$ time steps.

\item Play action $a=0$ in state $s$ for the next $N$ time steps to
  obtain an estimate $\h \mu_i^s$ for all $i \in [k]$.

\end{itemize}

\item Using the estimated losses for the states in $\sK$ and
  Equation~\eqref{eq:vertex-loss-comp-m-equal-2-main-app}, 
  %in Appendix~\ref{sec:app-proofs-mdp}, 
  run the oracle for the optimization \eqref{eq:opt-state} to obtain an approximately optimal state $\hat{s}$.
\item Move from current state to $\hat{s}$ and play action $a=0$ from $\hat{s}$ for the remaining time steps.
\end{enumerate}
}}
\end{center}
\caption{Online algorithm for $m = 2$ achieving $\tilde{O}(T^{2/3})$ pseudo-regret.}
\label{ALG:stochastic-mdp-app} 
\end{figure}

\begin{reptheorem}{thm:regred-mdp-m-equal-2}
Consider an $\MDP$ with losses in $[0, \B]$, a maximum fixing cost
$\cmax$, and correlations sets of size at most $m = 2$.  Let $\sK$ be
a cover of $\cC$ of size $r \leq 4n$, then, the algorithm of
Figure~\ref{ALG:stochastic-mdp-app} (same as Figure~\ref{ALG:stochastic-mdp}) achieves a pseudo-regret bounded by
$O(k r^{1/3} (\cmax +\B) (\log rkT)^{1/3} T^{2/3})$. Furthermore,
given access to the optimization oracle for \eqref{eq:opt-state}, the
algorithm runs in time polynomial in $k$ and $n = |\cC|$.
\end{reptheorem}
\begin{proof}
In each time step the maximum loss incurred by any criterion is bounded by $c+B$. Let $\{s_1, s_2, \dots, s_r\}$ be the states in $\sK$. During the exploration phase the algorithm stays in each state for $N$ time steps and incurs a total loss bounded by $kNr (c+B)$. During the exploration phase the algorithm moves from one state to another in at most $k$ steps and incurs a total additional loss of at most $rk^2(c+B)$. At any given state $s \in \sK$ and vertex $v_i$, after $N$ time steps we will, with probability at least $1-\delta$, an estimate of $\mu^s_i$ up to an accuracy of $2B\sqrt{\frac{\log 1/\delta}{N}}$. Setting $\delta = 1/(rk T^4)$ and using union bound, we have that at the end of the exploration phase, with probability at least $1 - \frac{1}{T^4}$, the algorithm will have estimate $\hat{\mu}^s_i$ for all $s \in \sK$ and $i \in [k]$ such that
\begin{align}
\hat{\mu}^s_i - \mu^s_i \leq 4B\sqrt{\frac{\log rkT}{N}}.
\end{align}

% $$
% \Big| \hat{\E}[\ell^{(i)} | s] - {\E}[\ell^{(i)} | s] \Big| \leq 4R\sqrt{\frac{\log rkT}{N}}.
% $$
Hence during the exploitation phase, with high probability, the algorithm will have the estimate for the expected loss of each state in $\sS$, i.e., $\sum_i \mu^s_i$ up to an error of $4kB \sqrt{\frac{\log rkT}{N}}$. Combining the above we get that the total pseudo-regret of the algorithm is bounded by
\begin{align*}
 \Reg(\cA) &\leq  kNr(c+B) + rk^2 (c+B) + \Big(1-\frac{1}{T^4} \Big)4kB T \sqrt{\frac{\log rkT}{N}} + \frac{1}{T^4}k(c + B)T.
\end{align*}
Setting $N=10 \frac{T^{2/3}( \log rkT)^{1/3}}{r^{2/3}}$ we get that 
$$
\Reg(\cA) \leq O(k r^{1/3} (c+B) (\log rkT)^{1/3} T^{2/3}).
$$
\end{proof}
\subsection{General case} 

The algorithm for the case of $m = 2$ naturally extends to arbitrary
correlation set sizes. 
%Here, we assume that we are given a collection $\calC$ of correlation sets and each vertex $i$ participates in at
%most $n$ sets. 
Overall the structure of the algorithm remains the
same where we pick a cover of $\calC$ and estimate the losses incurred
in states that belong to the cover. Using the estimated losses we are able to approximately estimate the loss of any vertex at any other
state. In order to do this we extend the definition of the cover as
follows. Given correlation sets of arbitrary size in $\cC$, a vertex $v_i$ may participate in many of them. We say that vertices $v_i$ and $v_j$ share a correlation set, if they appear together in a set in $\cC$. Consider the set of indices of all the vertices that $v_i$ shares a correlation set with. We partition this set into disjoint subsets such that no two vertices in different subsets share a correlation set. For a given vertex $v_i$, we denote this collection of disjoint subsets by $I_i$. For example, if $\cC$ contains sets $\{1,2\}$, $\{1,3\}$, and $\{1,4\}$, then, $I_1$ consists of the set $\{2,3,4\}$. On the other hand if $\cC$ contains sets $\{1,2,3\}, \{1,3,4\},$ and $\{1,6,7\}$ then,  $I_1$ consists of sets $\{2,3,4\}$ and $\{6,7\}$. For a given state $s$ and $J \in I_i$ we denote by $s(J)$ the vector $s$ restricted to indices in $J$. Notice that, in the worst case, $I_i$ will consist of a single set of size at most $\min(k-1, nm)$. However, for more structured cases (e.g, $m=2$) we expect $I_i$ to consist of subsets of small sizes.

% Fix an arbitrary ordering of the vertices in $\sG$. For any
% $i \in [k]$, and a state $s$, define $s_{-i}$ to be the vector $s$
% restricted to the indices that share a correlation set with $i$. Hence
% $s_{-i}$ is a vector of length at most $mn$.

% define $\tilde{C}_i = \cup_{j: i \in \sC_j} \sC_j$. Given a state $s \in \sS$ and a subset $T \subseteq [k]$, let $s_{T}$ denote the $|T|$-dimensional vector corresponding to the values of vertices in $T$ in state $s$. 
Given $i\in [k]$, $J \in I_i$, $b \in \{0,1\}$ and vectors ${u}_1, {u}_2$, we say
that $(i, b, J, {u}_1, {u}_2)$ is a dichotomy, if there exist two states
$s, s' \in \sS$ such that: (1) $s(J) = {u}_1, s'(J) = {u}_2$,
(2) $s(i) = b = s'(i)$, and (3) $s,s'$ agree in all other criteria. 
We call such a pair of states $s,s'$ an
$(i,b, J, u_1, u_2)$ pair. We next extend the definition of a cover as
follows. A subset $\sK \subseteq S$ is called a cover of $\calC$ if
for any valid dichotomy $(i, b, J, {u}_1, {u}_2)$, there exists an
$(i, b, J, {u}_1, {u}_2)$ pair $s,s' \in \sK$.
% exists two states $s$ and $s'$ in $\sK$ such that: (1) The states
% agree on all criteria except for those in $T$, (2)
% $s^0_i = b = s^1_i$, and (3)
% $s^0_{\tilde{T}_i \setminus \{i\}} = \bm{u}_1, s^1_{\tilde{T}_i
% \setminus \{i\}} = \bm{u}_2$.
In general, we will always have a cover of size at most $n
2^{mn}$. Similar to \eqref{eq:def-X-m=2}, for a valid dichotomy $(i,b,J,u_1, u_2)$, we define $X^{i,u_1, u_2}_{b,J}$ as
\begin{align}
    \label{eq:def-x-general-m}
    X^{i,u_1, u_2}_{b,J} \coloneqq \mu^s_i - \mu^{s'}_{i},
\end{align}
where $s,s' \in \sK$ is an $(i,b,J,u_1, u_2)$ pair.
Given the loss values in the states present in $\sK$, we can
estimate the loss of any other state using
Theorem~\ref{thm:phi-cover-m-general} stated below. 
% See
% Appendix~\ref{sec:app-proofs-mdp} for a proof.
\begin{theorem}
\label{thm:phi-cover-m-general}
Let $\sK$ be a cover for $\calC$.
% and let $s,s' \in \sK$ be an $(i,b,J,u_1, u_2)$ pair. Define $X^{i,u_1, u_2}_{b,J} \coloneqq \mu^s_i - \mu^{s'}_{i}$. 
Then, for any state $s \in \sS$ and any $i \in [k]$ with $s(i) = b$, we have:
\begin{align}
% \label{eq:vertex-loss-comp-m-general}
%   \mu_i^s & = \mu_i^{s'} +  X^{i, s_{-i}}_{b} - X^{i, s'_{-i}}_{b},
% \end{align}
\label{eq:vertex-loss-comp-m-general}
  \mu_i^s & = \mu_i^{s''} +  \sum_{J \in I_i} X^{i, s(J), s''(J)}_{b,J}
\end{align}
Here $s''$ is any state in $\sK$ with $s''(i)=b$. 
% Furthermore, the quantities $X^{i, s_{-i}}_{b}$ can be computed efficiently.
\end{theorem}
\begin{proof}
Let $s,s' \in \sK$ be an $(i,b,J, u_1, u_2)$ pair. When we move from state $s$ to $s'$, the only difference between the expected losses incurred by vertex $v_i$ comes from the configuration of the vertices in $J$. Hence there at at most $2^{|J|+1}$ distinct parameters governing the expected loss incurred by vertex $i$ in a given state $s$ due to the configuration of the vertices in $J$. Denoting these parameters by $\gamma^{b,s(J)}_{i,J}$ we have
\begin{align*}
    \mu^{s'}_i - \mu^s_i = \gamma^{b,s'(J)}_{i,J} - \gamma^{b,s(J)}_{i,J} \coloneqq  X^{i,s'(J), s(J)}_{b,J}.
\end{align*}
Given the loss values for the states in the cover $\sK$, we can estimate the quantities $X^{i,s(J), s''(J)}_{b,J}$.

Next, for an arbitrary state $s$ such that $s(i)=b$, let $s'' \in \sK$ be such that $s''(i)=b$. We have
\begin{align*}
    \mu^s_i &= \mu^{s''}_i + \sum_{J \in I_i} \gamma^{b,s(J)}_{i,J} - \gamma^{b,s''(J)}_{i,J}\\
    &= \sum_{J \in I_i} X^{i, s(J), s''(J)}_{b,J}.
\end{align*}
% Let $s,s' \in \Phi$ be an $(i,b,u_1, u_2)$ and for ease of notation let $C_1, C_2, \dots, C_t$ be the correlation sets that $i$ participates in. The loss incurred by vertex $v_i$, at any given state $s$ solely depends on the configuration of the vertices in $\cup_j C_j$. Hence, the parameters in $\bm{\theta}$ that affect the loss incurred by vertex $i$ in any state $s$ can be denoted by $\theta^{i,s_{-i}}_1, \dots \theta^{i,s_{-i}}_t$.\footnote{{\tt YM: What is the difference from the previous definition of $\theta^s_j$?}} Notice that there are at most $2^{mt+1}$ such parameters. When we move from $s$ to $s'$ the only difference in the expected loss incurred by vertex $i$ comes from the change in configuration of the vertices in $\cup_j C_j$. Hence we can write
% \begin{align*}
% \E[\ell^{(i)}|s] - \E[\ell^{(i)}|s'] &= \sum_{j=1}^t \theta^{i,s_{-i}}_j - \theta^{i,s'_{-i}}_j\\
% & \coloneqq X^{i,s_{-i}}_b - X^{i,s'_{-i}}_b.
% \end{align*}
% Since $\Phi$ is cover for $\cC$ we can estimate $X^{i,s_{-i}}_b - X^{i,s'_{-i}}_b$ for all configurations $s_{-i}$ and $s_{-i}$ by estimating expected loss values for states in the cover.
% \footnote{{\tt YM: Seems like $X^{i, s_{-i}}_{b}=\mu^s_i$ so really the ONLY thing tio show is how to recover the parameters using the cover. I did not find any hint. What am I missing?}}
%
% using     \label{eq:vertex-loss-comp-m-general-app} this lets us estimate the expected loss values for vertex $v_i$ in any other state.
\end{proof}
For general correlation sets with each vertex participating in at most
$n$ sets, we use \eqref{eq:vertex-loss-comp-m-general} instead of
\eqref{eq:vertex-loss-comp-m-equal-2-main-app} to estimate losses in step 4 of
the algorithm in Figure~\ref{ALG:stochastic-mdp-app}. The algorithm for general $m$ 
is described in Figure~\ref{ALG:stochastic-mdp-general} and has the following associated regret guarantee. The proof is identical to the proof of Theorem~\ref{thm:regred-mdp-m-equal-2}.
\begin{figure}[t]
\begin{center}
\fbox{\parbox{1\textwidth}{
{\bf Input:} The graph $\sG$, correlation sets $\calC$, fixing costs $c_i$.
\begin{enumerate}   
\item Pick a cover $\sK =  \set{s_1, s_2, \dots, s_r}$ of $\cC$. 

\item Let $N=10 \frac{T^{2/3}( \log rkT)^{1/3}}{r^{2/3}}$.

\item For each state $s \in \sK$ do:

\begin{itemize}

\item Move from current state to $s$ in at most $k$ time steps.

\item Play action $a=0$ in state $s$ for the next $N$ time steps to
  obtain an estimate $\h \mu_i^s$ for all $i \in [k]$.

\end{itemize}

\item Using the estimated losses for the states in $\sK$ and
  Equation~\eqref{eq:vertex-loss-comp-m-general}, 
  %in Appendix~\ref{sec:app-proofs-mdp}, 
  run the oracle for the optimization \eqref{eq:opt-state} to obtain an approximately optimal state $\hat{s}$.
\item Move from current state to $\hat{s}$ and play action $a=0$ from $\hat{s}$ for the remaining time steps.
\end{enumerate}
}}
\end{center}
\caption{Online algorithm for general $m$ achieving $\tilde{O}(T^{2/3})$ pseudo-regret.}
\label{ALG:stochastic-mdp-general} 
\end{figure}
%in Appendix~\ref{sec:app-proofs-mdp}.
\begin{reptheorem}{thm:regred-mdp-m-general-m}
Consider an $\MDP$ with losses bounded in $[0, \B]$ and maximum cost of fixing a vertex being $\cmax$. Given correlations sets $\calC$ of size at most $m$, and a cover $\sK$ of $\calC$ of size $r \leq n 2^{mn}$, the algorithm in Figure~\ref{ALG:stochastic-mdp-general} achieves a pseudo-regret bounded by $O(k r^{1/3} (\cmax +\B) (\log rkT)^{1/3} T^{2/3})$. Furthermore, given access to the optimization oracle for \eqref{eq:opt-state} the algorithm runs in time polynomial in $k$, $n = |\calC|$ and $r = |\sK|$.
% Given graph $\sG$, arbitrary correlation sets $\calC$, and a cover
% $\sK$ of $\calT$ of size $r \leq n 2^{mt}$, the algorithm in
% Figure~\ref{ALG:stochastic-mdp} achieves a pseudo-regret bounded by
% $O(k (r \log r)^{1/3} (b + c_{\text{max}})^{1/3}
% T^{2/3})$. Furthermore, given access to the optimization oracle for
% \eqref{eq:opt-state} the algorithm runs in time polynomial in $k$,
% $r$, and $n = |\calT|$.
\end{reptheorem}

\section{Beyond $T^{\frac{2}{3}}$ regret}
\label{sec:ucb-app}

In this section, we present algorithms for our 
problem that achieve $\tilde{O}(\sqrt{T})$ regret,
first in the case $m = 1$, next for any $m$, under the natural assumption that each criterion does not
participate in too many correlations sets.

Let us first point out that our problem can be
cast as an instance of the stochastic multi-armed 
bandit problem with switching costs, where each
state $s$ is viewed as an arm and where the cost
of transitions from state $s$ to state $s'$ is the 
switching cost between $s$ and $s'$. For the instance of this problem with identical switching costs, \citet{cesa2013online}[Appendix A] gave an algorithm achieving expected regret $\tilde{O}(\sqrt{T})$, via an arm-elimination technique with at most
$O(\log \log T)$ switches. However, naturally, the regret guarantee and the time complexity of that algorithm depend on the number of arms, which in our case is exponential ($2^k$). We will show here that, in most realistic instances of our model, we can achieve $\tilde{O}(\sqrt{T})$ regret efficiently.

\ignore{
if exponential time (in the size of the graph) is allowed then we can achieve an expected regret of $\tilde{O}(2^k \sqrt{T})$ by casting our problem as that of learning in a multi-armed bandit setting with $2^k$ arms and i.i.d.\ stochastic losses. The fixing cost of the criteria then corresponds to the cost of switching between arms. \citet{cesa2013online}[Appendix A] provide an algorithm with an expected regret of $\tilde{O}(2^k \sqrt{T})$. However, the computational of exponential in $k$ is prohibitively large. 
}

We first consider the case where the correlations sets in $\cC$ are of size one ($m = 1$). In this case, the parameter vector ${\bm \theta}$ can be described using the following $2k$ parameters: for each $i \in [k]$, let $\gamma^0_i$ denote the expected loss incurred by criterion $i$ when it is unfixed and $\gamma^1_i$ its expected loss when it is fixed. In this case, the cover $\sK$ is of size $k + 1$ and includes the all-zero state, as well as $k$ states corresponding to the indicator vectors of the $k$ vertices. Our algorithm is similar to the UCB algorithm for multi-armed bandits \cite{auer2002finite} and maintains optimistic estimates of the parameters. For every vertex $i$, we denote by $\tau^0_{i, t}$ the total number of time steps up to $t$ (including $t$) during which the vertex $v_i$ is in an unfixed position and by $\tau^1_{i, t}$ the total number of times steps up to $t$ during which vertex $v_i$ is in a fixed position. Fix $\delta \in (0, 1)$ and let $\hat{\gamma}^b_{i, t}$ be the empirical expected loss observed when vertex $v_i$ is in state $b$, for $b \in \set{0, 1}$. Our algorithm maintains the following optimistic estimates at each time step $t$,
\begin{align}
    \label{eq:optimistic-estimates-main-app}
    \tilde{\gamma}^b_{i,t} = \hat{\gamma}^b_{i,t} - 10B\sqrt{\frac{\log (kT/\delta)}{\tau^b_{i,t}}}.
\end{align}

To minimize the fixing cost incurred when transitioning from one state to another, our algorithm works in episodes. In each episode $h$, the algorithm first uses the current optimistic estimates to query the optimization oracle and determine the current best state $s$. Next, it remains at state $s$ for $t(h)$ time steps before querying the oracle again. The number of time steps $t(h)$ will be chosen carefully to avoid incurring the fixing costs too often. 
The algorithm is described in Figure~\ref{ALG:stochastic-mdp-m=1-ucb-main-app} (same as Figure~\ref{ALG:stochastic-mdp-m=1-ucb-main} in main body ).
We will prove that it benefits from the following 
regret guarantee.

\begin{figure}[t]
\begin{center}
\fbox{\parbox{1\textwidth}{
{\bf Input:} graph $\sG$, correlation sets $\calC$, fixing costs $c_i$.
\begin{enumerate}   

\item Let $\sK$ be the cover of size $k + 1$ that includes the all zeros state and the states corresponding to indicator vectors of the $k$ vertices. 

\item Move to each state in the cover once and update the optimistic estimates according to \eqref{eq:optimistic-estimates-main-app}.

\item For episodes $h = 1,2, \dots$ do:

\begin{itemize}

\item Run the optimization oracle \eqref{eq:opt-state} with the optimistic estimates as in \eqref{eq:optimistic-estimates-main-app} to get a state $s$.

\item Move from current state to state $s$. Stay in state $s$ for $t(h)$ time steps and update the corresponding estimates using \eqref{eq:optimistic-estimates-main-app}. Here $t(h) = \min_i \tau^{s(i)}_{i,t_h}$ and $t_h$ is the total number of time steps before episode $h$ starts.
\end{itemize}

\end{enumerate}
}}
\end{center}
\caption{Online algorithm for $m = 1$ with $\tilde{O}(\sqrt{T})$ regret.}
\label{ALG:stochastic-mdp-m=1-ucb-main-app} 
\end{figure}

\begin{reptheorem}{thm:regred-mdp-m-equal-1-ucb}
Consider an $\MDP$ with losses bounded in $[0, B]$ and maximum cost of
fixing a vertex being $c$. Given correlations sets $\cC$ of
size one, the algorithm of Figure~\ref{ALG:stochastic-mdp-m=1-ucb-main-app} (same as Figure~\ref{ALG:stochastic-mdp-m=1-ucb-main}) achieves a
pseudo-regret bounded by
$O(k^2 (c + B)^{2} \sqrt{T} \log T)$. Furthermore,
given access to the optimization oracle for \eqref{eq:opt-state}, the
algorithm runs in time polynomial in $k$.
\end{reptheorem}
\begin{proof}
We first bound the total number of different states visited by the algorithm. Initially the algorithm visits $k+1$ states in the cover. After that, each time the optimization oracle returns a new state $s$, by the definition of $t(h)$, the number of time steps where some vertex is in a $0$ or $1$ position is doubled. Hence, at most $O(k \log T)$ calls are made to the optimization oracle. Noticing that one can move from one state to another in at most $k$ time steps, the total loss incurred during the switching of the states is bounded by $O(k^2 (c+B) \log T)$. 
% \footnote{{\tt YM: I do not see how you get $k^3$. You have $k\log T $ switches, and each costs at most $k(\Gamma+R)$}}

For $\epsilon >0$ to be chosen later, we consider the episodes where the algorithm plays a state $s$ with expected loss at most $\epsilon$ more than that of the best state $s^*$. The total expected regret accumulated in these {\em good} episodes is at most $\epsilon T$. We next bound the expected regret accumulated during the bad episodes. 

From Hoeffding's inequality we have that for any time $t$, with probability at least $1-\frac{\delta}{T^3}$, for all $i \in [k],b \in \{0,1\}$,
\begin{align}
    \label{eq:concentration-of-estimates}
    \tilde{\gamma}^b_{i,t} + 20B\sqrt{\frac{\log (kT/\delta)}{\tau^b_{i,t}}} \geq \gamma^b_i \geq \tilde{\gamma}^b_{i,t}.
\end{align}
Let $G$ be the good event that \eqref{eq:concentration-of-estimates} holds for all $t \in [1,T]$. Conditioned on $G$ we also have that for any state $s$ and vertex $i$
\begin{align}
\label{eq:bound-on-optimistics-loss-per-state}
\mu^s_i \geq \tilde{\mu}^{s}_i, 
\end{align}
% The above implies that conditioned iexpected loss of a given state of the sum of the expected losses of the parameters in ${\bm \theta}$, conditioned on the event above we also have that for any state $s$,
% \begin{align}
% \label{eq:bound-on-optimistics-loss-per-state}
% \E[\ell | s] \geq \tilde{\ell}(s), 
% \end{align}
where $\tilde{\mu}^{s}_i$ is the estimated loss using the optimistic estimates. We will bound the expected regret accumulated in the bad episodes conditioned on the event $G$ above.

In order to do this we define certain key quantities. Consider a particular trajectory $\mathcal{T}$ of $T$ time steps executed by the algorithm. Furthermore, let $\mathcal{T}$ be such that the good event in \eqref{eq:concentration-of-estimates} holds during the $T$ time steps. We associate the following random variables with the trajectory. Let $N_\epsilon$ be the total number of time steps spent in bad episodes. Furthermore, let $\Reg_\epsilon$ be the total accumulated regret during these time steps. Then it is easy to see that $\E[\Reg_\epsilon | G] > \epsilon N_\epsilon$. For each vertex $v_i$ and $b \in \{0,1\}$ we define $\tau_\epsilon(i,b)$ to be the total number of time steps that vertex $v_i$ spends in bad episodes in position $b$ and $\tau_\epsilon(i,b,t)$ to be the total number of time steps spent in bad episodes up to time step $t$. Notice that
% \footnote{{\tt YM: Are you missing a "2" due to the $b\in\{0,1\}$}}
\begin{align}
\label{eq:sum-of-n-eps}
    \sum_b \sum_i \tau_\epsilon(i,b) \leq 2k N_\epsilon.
\end{align}

 Consider a particular bad episode $h$ and let $s$ be the state returned by the optimization oracle during that episode. Then conditioned on the good event $G$, the total regret $\Reg_{h}$ accumulated during episode $h$ satisfies
%  \footnote{{\tt YM: What is $\theta^{s(i)}_{i,t_h}$ do you mean $\tilde{\theta}^{s(i)}_{i,t_h}$}}
 \begin{align*}
    \E[{\Reg}_{h}|\mathcal{T}] &= \sum_i \big(\mu^s_i - \mu^{s^*}_i \big) t(h)\\
    &\leq \sum_i \big(\mu^s_i - \tilde{{\mu}}^{s^*}_i \big) t(h) & \big(\text{from} \eqref{eq:bound-on-optimistics-loss-per-state}\big)\\
    &\leq \sum_i \big(\mu^s_i - \tilde{{\mu}}^{s}_i \big) t(h) & \big(\text{since $s$ is best state according to the optimistic losses} \big)\\
    &\leq \sum_{i} \big(\gamma^{s(i)}_{i} - \tilde{\gamma}^{s(i)}_{i,t_h} \big) t(h) \\
    &\leq \sum_i 20B \sqrt{\frac{\log (kT/\delta)}{\tau^b_{i,t_h}}} t(h). & \big(\text{from \eqref{eq:optimistic-estimates}} \big)
\end{align*}
In the above inequality, the expectation is taken over the loss distribution for each vertex during states visited in the trajectory $\mathcal{T}$.
% \begin{align*}
%     {\Reg}_{h} &= \big(\E[\ell | s] - \E[\ell | s^*] \big) \tau_h\\
%     &\leq \big(\E[\ell | s] - \tilde{\ell}(s^*) \big) \tau_h & \big(\text{from} \eqref{eq:bound-on-optimistics-loss-per-state}\big)\\
%     &\leq \big(\E[\ell | s] - \tilde{\ell}(s) \big) \tau_h & \big(\text{since $s$ is best state according to the optimistic losses} \big)\\
%     &\leq \sum_{i} \big(\theta^{s(i)}_{i} - \theta^{s(i)}_{i,t_h} \big) \tau_h \\
%     &\leq \sum_i 20R \sqrt{\frac{\log (kT/\delta)}{n^b_{i,t_h}}} \tau_h. & \big(\text{from \eqref{eq:optimistic-estimates}} \big)
% \end{align*}
% \footnote{{\tt YM: I do not understand the last inequality. $n^b_{i,t_h}$ is a random variable, so what it is doing here? Actually $\tilde{\ell}$ and $\tilde{\theta}$ are also random variables!}}

%Denoting by $n_\epsilon(i,b,t_h)$ to be the total number of time steps vertex $i$ spends in bad episodes in position $b$ up to time step $t_h$. 
Since $\tau^b_{i,t_h} \geq \tau_\epsilon(i,b,t_h)$ we have
% \footnote{{\tt YM: $n_\epsilon(i,b,t_h)$ is a random variable. Do you the expected total number of steps, or with high probability? I am gettiung confused in this proof.}} 
we have that
\begin{align*}
    \E[{\Reg}_{h} | \mathcal{T}] &\leq \sum_i 20B \sqrt{\frac{\log (kT/\delta)}{\tau_\epsilon(i,b,t_h)}} t(h).
\end{align*}
Summing over bad episodes, the total expected regret in bad episodes can be bounded by
\begin{align}
\label{eq:reg-intermediate}
    \E[{\Reg}_{\epsilon} | \mathcal{T}] &\leq \sum_i \sum_b \sum_{h: h \text{ is bad }} 20B \sqrt{\frac{\log (kT/\delta)}{\tau_\epsilon(i,b,t_h)}} t(h).
\end{align}
Notice that $\tau_\epsilon(i,b,t_h) = \sum_{h' < h: h' \text{ is bad}} t(h')$. Furthermore, we know that (\cite{jaksch2010near}) for any sequence $z_1, z_2, \dots, z_h$ of non-negative numbers such that $z_i \geq 1$,
\begin{align}
\label{eq:bounded-sequence}
   \sum_{i=1}^h \frac{z_i}{\sqrt{\sum_{j=1}^{i-1} z_j}} &\leq (1+\sqrt{2})\sqrt{\sum_{i=1}^h z_i}.
\end{align}
From \eqref{eq:bounded-sequence} we get:
\begin{align*}
    \sum_{h: h \text{ is bad }}  \frac{t(h)}{\sqrt{\tau_\epsilon(i,b,t_h)}} &\leq \sqrt{\tau_\epsilon(i,b)}. 
\end{align*}
Substituting into \eqref{eq:reg-intermediate} we get that
\begin{align*}
    \E[{\Reg}_{\epsilon} | \mathcal{T}] &\leq \sum_i \sum_b 20B \sqrt{{\log (kT/\delta)}} \sqrt{\tau_\epsilon(i,b)}.
\end{align*}
Using \eqref{eq:sum-of-n-eps} we have that the above expected regret is maximized when $\tau_\epsilon(i,b)$ are equal, thereby implying 
\begin{align*}
    \E[{\Reg}_{\epsilon}|\mathcal{T}] &\leq 20Bk \sqrt{{\log (kT/\delta)}} \sqrt{N_\epsilon}.
\end{align*}
Using the fact that $\E[\Reg_\epsilon | G] > \epsilon N_\epsilon$ we get that conditioned on $G$,
\begin{align*}
    N_\epsilon \leq \frac{400 B^2 k^2 \log(kT/\delta)}{\epsilon^2}.
\end{align*}
Combining trajectories $\mathcal{T}$ where the good event $G$ holds, we get that the total expected regret accumulated in the bad episodes satisfies
\begin{align*}
    \E[{\Reg}_{\epsilon}|G] &\leq 20Bk \sqrt{{\log (kT/\delta)}} \sqrt{N_\epsilon}\\
    &\leq 400B^2 k^2 \frac{\log(kT/\delta)}{\epsilon}.
\end{align*}
Combining the above with the total expected regret accumulated in the good episodes, the loss of moving to different states, and the probability of good event $G$ not holding, we get
\begin{align*}
    {\Reg}(\cA) &\leq 400B^2 k^2 \frac{\log(kT/\delta)}{\epsilon} + \epsilon T + \frac{k(c+B)}{T^3} + O(k^2 (c+B) \log T).
\end{align*}
Setting $\epsilon = \frac{1}{\sqrt{T}}$  and $\delta = \frac{1}{T^4}$, we have the final bound
\begin{align*}
    {\Reg}(\cA) &\leq O\big((c+B)^2 k^2  \sqrt{T} \log(T)\big).
\end{align*}
\end{proof}

The algorithm of Figure~\ref{ALG:stochastic-mdp-m=1-ucb-main-app} can be extended to higher $m$ values, assuming that each vertex does not participate in too many correlation sets. \ignore{We say that vertices $v_i$ and $v_j$ are correlated if they appear together in some correlation set in $\cC$.}If a vertex $v_i$ appears in at most $O(\log k)$ correlation sets, then the total loss incurred by vertex $v_i$ in any state depends on the position of $v_i$ and every other vertex that it is correlated with. Hence the total loss incurred by vertex $v_i$ depends on an $O(m \log k)$-dimensional vector. For every configuration ${\bm b}$ of this vector, we associate with each vertex $v_i$, parameters $\gamma^{\bm b}_i$. Notice that there are at most $O(k^m)$ such parameters. Each parameter is in turn a sum of a subset of the parameters in ${\bm \theta}$. Notice that in this case the size of the cover $\sK$ is upper bounded by $O(k^{m+1})$. Our algorithm for higher $m$ values is similar to the one for $m = 1$, but instead maintains optimistic estimates of the parameters $\gamma^{\bm b}_i$ via
\begin{align}
    \label{eq:optimistic-estimates-higher-m}
    \tilde{\gamma}^{\bm b}_{i,t} = \hat{\gamma}^{\bm b}_{i,t} - 10B\sqrt{m\frac{\log (kT/\delta)}{\tau^{\bm b}_{i,t}}}.
\end{align}
Here $\tau^{\bm b}_{i,t}$ is the total time spent up to and including $t$ where the vertex $i$ and the vertices that it is correlated with are in configuration ${\bm b}$. Similarly, for a given state $s$, we will denote by ${\bm s}(i)$, the configuration of the vertex $i$ and the vertices that it is correlated with. The algorithm is sketched below
\begin{figure}[h]
\begin{center}
\fbox{\parbox{1\textwidth}{
{\bf Input:} The graph $\sG$, correlation sets $\calC$, fixing costs $c_i$.
\begin{enumerate}   
\item Let $\sK$ be the cover of size $O(k^{m+1})$. 

\item Move to each state in the cover once and update the optimistic estimates according to \eqref{eq:optimistic-estimates-higher-m}.

\item For episodes $h = 1,2, \dots$ do:

\begin{itemize}

\item Run the optimization oracle \eqref{eq:opt-state} with the optimistic estimates as in \eqref{eq:optimistic-estimates-higher-m} to get a state $s$.

\item Move from current state to state $s$. Stay in state $s$ for $t(h)$ time steps and update the corresponding estimates using \eqref{eq:optimistic-estimates-higher-m}. Here $t(h) = \min_i \tau^{{\bm s}(i)}_{i,t_h}$ and $t_h$ is the total number of time steps before episode $h$ starts.
\end{itemize}

\end{enumerate}
}}
\end{center}
\caption{Online algorithm for higher $m$.}
\label{ALG:stochastic-mdp-m=higher--ucb} 
\end{figure}

For $m \geq 1$, we obtain the following guarantee.
\begin{theorem}
\label{thm:regred-mdp-m-higher-ucb}
Consider an $\MDP$ with losses bounded in $[0,B]$ and maximum cost of
fixing a vertex being $c$. Given correlations sets $\cC$ of size at most $m$ such that each vertex participates in at most $O(\log k)$ sets, the the algorithm in Figure~\ref{ALG:stochastic-mdp-m=higher--ucb} achieves a
pseudo-regret bounded by
$O(m k^{2m+2} (c+B)^{2} \sqrt{T} \log T)$. Furthermore,
given access to the optimization oracle for \eqref{eq:opt-state}, the
algorithm runs in time polynomial in $O(k^{m + 1})$.
\end{theorem}
\begin{proof}
The proof is very similar to the proof of Theorem~\ref{thm:regred-mdp-m-equal-1-ucb}. Since each time the optimization oracle is called the time spent in some configuration ${\bm s}(i)$ is doubled, we get that the total number of calls to the optimization oracle are bounded by $O(k^m \log T)$. Hence the total loss incurred during the exploration phase can be bounded by $O(k^m (c+B) \log T)$. Let $G$ be the good event that \eqref{eq:optimistic-estimates-higher-m} holds for all $t \in [1,T]$.

As before, the loss incurred during good episodes is bounded by $\epsilon T$. Define $\tau_\epsilon(i, {\bm b})$ to be the total time that vertex $i$ and vertices that it is correlated with spend in configuration ${\bm b}$ during bad episodes. Then analogous to \eqref{eq:sum-of-n-eps} we have
\begin{align*}
    \sum_{{\bm b}} \sum_i \tau_\epsilon(i, {\bm b}) &\leq O(k^m) N_\epsilon.
\end{align*}
For a trajectory $\mathcal{T}$ where the good event $G$ holds,
the total expected regret in bad episodes can be bounded as
% \footnote{{\tt YM: I think you are missing $\sum_b$ in (17). Again, there is a mismatch between expectation (expected regret) and the random variables $n_\epsilon(i,{\bm b},t_h)$ and $N_\epsilon$}}
\begin{align}
\label{eq:reg-intermediate-general}
    \E[{\Reg}_{\epsilon}|\mathcal{T}] &\leq \sum_i \sum_b \sum_{h: h \text{ is bad }} 20B \sqrt{m \frac{\log (kT/\delta)}{\tau_\epsilon(i,{\bm b},t_h)}} t(h)\\
    &\leq \sum_i \sum_{{\bm b}} 20B \sqrt{m \log(kT/\delta)} \sqrt{\tau_\epsilon(i, {\bm b})}\\
    &\leq O(B k^{m+1})\sqrt{m \log(kT/\delta)} \sqrt{N_\epsilon}.
\end{align}
Using the fact that $\E[\Reg_\epsilon | \mathcal{T}] > \epsilon N_\epsilon$ we get that for a trajectory where the event $G$ holds,
\begin{align*}
    N_\epsilon \leq \frac{O(R^2 k^{2m+2} m \log(kT/\delta))}{\epsilon^2}.
\end{align*}
Hence we get that conditioned on the good event $G$, the total expected regret accumulated in the bad episodes is at most
\begin{align*}
    \E[{\Reg}_{\epsilon}|G] &\leq  O \big(R^2 m k^{2m+2} \frac{\log(kT/\delta)}{\epsilon} \big).
\end{align*}

Combining the above with the total expected regret accumulated in the good episodes, the loss of moving to different states, and the probability of the event $G$ not holding we get
\begin{align*}
    {\Reg}(\cA) &\leq O\big(B^2 m k^{2m+2} \frac{\log(kT/\delta)}{\epsilon}\big) + \epsilon T + \frac{k(c+B)}{T^3} + O(k^m \log T).
\end{align*}
Setting $\epsilon = \frac{1}{\sqrt{T}}$  and $\delta = \frac{1}{T^4}$, we have the final bound
\begin{align*}
    {\Reg}(\cA) &\leq O\big((c+B)^2 m  k^{2m+2}  \sqrt{T} \log(T)\big).
\end{align*}
\end{proof}
An important corollary of the above is the following
\begin{corollary}
\label{cor:degree-d-graph}
If $\sG$ is a constant degree graph with correlation sets consisting of subsets of edges in $\sG$, then there is a polynomial time algorithm that achieves a pseudo-regret bounded by $O(k^6 (c + B)^{2} \sqrt{T} \log T)$.
\end{corollary}
%See Appendix \ref{sec:app-proofs-mdp} for the proofs.

\section{Adversarial Setting}
\label{sec:adversarial-app}

In this section we provide the proof of Theorem~\ref{thm:adversarial-online} restated below.
\begin{reptheorem}{thm:adversarial-online}
  Let $\sG$ be a graph with fixing costs at least one. Then, the
  algorithm of Figure~\ref{ALG:adversarial} achieves a competitive
  ratio of at most $2\B + 4$ on any sequence of complaints with loss values
  in $[0, \B]$.
\end{reptheorem}      
\begin{proof}
Recall that $\ell_{i(t)}$ denotes the loss incurred by vertex $v_i$ at time $t$. We divide this loss into the amount that was used to reduce the $\kappa_j$ value of one its neighbors and the rest. Formally, for every edge $(i,j)$ we define $\delta^t_{i \to j}$ as follows. If in time step $t$, the complaint arrived for vertex $i$ and step 2(b) was executed to reduce $\kappa_j$ by $\Delta$, then we define $\delta^t_{i \to j} = \Delta$. Otherwise we define $\delta^t_{i \to j}$ to be zero. We also define
\begin{align}
\label{eq:def-ii}
    \delta^t_{i \to i} = \ell_{i(t)} - \sum_{j \in N(i)} \delta^t_{i \to j}.
\end{align}
If vertex $v_i$ is fixed $f_i$ times during the course of the algorithm then we have that the total loss incurred by the algorithm can be written as
\begin{align}
\label{eq:def-lossA}
    \text{Loss}(\cA) &= \sum_{i=1}^k  f_i c_i + \sum_{i=1}^k \sum_{t=1}^T \big( \delta^t_{i \to i} + \sum_{j \in N(i)} \delta^t_{i \to j} \big).
\end{align}
Next we notice that each time a vertex $v_i$ is fixed it accumulates a value of $\kappa_i = c_i$. Furthermore, the total loss incurred by vertices as a result of executing step 2(b) is upper bounded by the total $\kappa$ value accumulated. Hence we have
\begin{align}
    \label{eq:kappa-bound}
    \sum_{t=1}^T \sum_{i=1}^k \sum_{j \in N(i)} \delta^t_{i \to j} &\leq \sum_{i=1}^k f_i c_i.
\end{align}
Substituting into \eqref{eq:def-lossA} we have
\begin{align}
\label{eq:def-lossB}
    \text{Loss}(\cA) &\leq  \sum_{i=1}^k  2f_i c_i + \sum_{i=1}^k \sum_{t=1}^T \delta^t_{i \to i}.
\end{align}
Next we bound the above loss for each vertex separately. For a given vertex $v_i$ that is fixed $f_i$ times by the algorithm, we can divide the time steps into $f_i+1$ intervals consisting of an interval $I_0$ starting from $t=0$ upto (and including) the first time $v_i$ is fixed. The next $f_i$ intervals correspond to the time spent by $v_i$ between two successive fixes. Denoting these intervals as $I_0, I_1, \dots$ we have that
\begin{align}
    \label{eq:vertex-loss}
    2f_i c_i + \sum_{i=1}^k \sum_{t=1}^T \delta^t_{i \to i} &= \sum_{t \in I_0} \delta^t_{i \to i} + \sum_{t \in I_r} (2c_i + \delta^t_{i \to i}).
\end{align}
Next we compare the above to the loss incurred by {\OPT} for vertex $v_i$. Let $\ell^*_{i(t)}$ be the loss incurred by {\OPT} for vertex $v_i$ at time $t$. We will denote by $s^*_t$ the state of the vertices at time $t$ according to {\OPT}. 
% Furthermore, for a complaint $(i,\ell_i)$ arriving at time $t$, if $v_i$ is unfixed in {\OPT} and the algorithm executes step 2(b) to increase $\kappa_j$ by $\Delta$, for $j \in N(i)$, we also increase $\ell^*_{j(t)}$ by $\Delta$. Notice that by counting in this way  we increase the loss incurred by {\OPT} by at most a factor of $2$. Let $\tilde{\ell}_{i(t)}$ be this new loss incurred by {\OPT} for a vertex $v_i$ at time $t$.

We instead redefine the loss incurred by {\OPT} for vertex $v_i$ at time $t$ to be
\begin{align}
    \label{eq:def-loss-opt}
    \tilde{\ell}_{i(t)} = \ell^*_{i(t)} + \sum_{j \in N(i)} \delta^t_{j \to i} \1(s^*_t(j)=0).
\end{align}
Notice that 
$$
\sum_{i \in N(j)} \delta^t_{j \to i} \1(s^*_t(j)=0) \leq \ell^*_{j(t)}.
$$
Hence we get that
\begin{align}
\sum_{i=1}^k \sum_{t=1}^T \tilde{\ell}_{i(t)} &\leq \sum_{i=1}^k \big( \sum_{t=1}^T \ell^*_{i(t)} + \sum_{j \in N(i)} \ell^*_{j(t)}\big)\\
&\leq 2 \cdot \text{Loss}(\text{\OPT}).
\end{align}
% Hence by counting $\tilde{\ell}_{i(t)}$ instead of $\ell^*_{i(t)}$ the total loss incurred by {\OPT} remains unchanged. 
Next we consider each interval in \eqref{eq:def-lossB} separately. For any interval $I_r$ we have that 
\begin{align}
    \label{eq:bound-on-ii}
    \sum_{t \in I_r} \delta^t_{i \to i} &\leq B c_i.
\end{align}
This is because after incurring a loss of more than $c_i$, any additional loss incurred by $v_i$ is due to step 2(b), since otherwise step 2(c) will be executed and $v_i$ will be fixed.

Next consider interval $I_0$. The loss incurred by the algorithm on vertex $v_i$ equals $\sum_{t \in I_0} \delta^t_{i \to i} \leq B c_i$. Either {\OPT} fixes $v_i$ at least once during this interval or incurs the total loss. Either way we have that the loss incurred by {\OPT} is at least 
\begin{align}
\label{eq:loss-1}    
\min \big(c_i, \sum_{t \in I_0} \delta^t_{i \to i} \big) &\geq \frac{\sum_{t \in I_0} \delta^t_{i \to i}}{B}.
\end{align}

Next consider an interval $I_r$ between two successive fixes. The loss incurred by the algorithm for vertex $v_i$ during this interval is at most 
$$
\sum_{t \in I_r} \delta^t_{i \to i} + 2c_i \leq (B+2)c_i.
$$
If {\OPT} fixes $v_i$ at least once during this interval then it incurs a cost of $c_i$. If $v_i$ remains unfixed in {\OPT} during the course of the interval then {\OPT} incurs a loss of at least $c_i$. This is because vertex $v_i$ went from being unfixed to fixed during the second half of the  interval and hence a total loss of at least $c_i$ must have arrived for the vertex $v_i$ during this interval. 

Finally, suppose vertex $v_i$ is fixed in {\OPT} before the start of the interval and remains so throughout. Since $v_i$ goes from being fixed to unfixed during the first half of the interval, we must have $\sum_{t \in I_r}\sum_{j \in N(i)} \delta^t_{j \to i} \geq c_i$. Furthermore, since $v_i$ is fixed by {\OPT} during this interval, {\OPT} must incur a loss on all neighbors of $j$. In particular, from \eqref{eq:def-loss-opt} we have
\begin{align}
    \sum_{t \in I_r} \tilde{\ell}_{i(t)} &\geq \sum_{t \in I_r}\sum_{j \in N(i)} \delta^t_{j \to i} \1(s^*_t(j)=0)\\
    &\geq c_i.
\end{align}

%during this interval and hence $\tilde{\ell}_{i(t)} \geq c_i$. 
In either of the three cases we have that the loss $\sum_{t \in I_r} \tilde{\ell}_{i(t)}$ incurred by {\OPT} is at least a $1/(B+2)$ fraction of the loss incurred by the algorithm. Summing over all the vertices and the corresponding intervals, we get that the total loss incurred by the algorithm can be bounded by 
$$
\text{Loss}(\cA) \leq (B+2)\sum_{t=1}^T \sum_{i=1}^k \tilde{\ell}_{i(t)} \leq 2(B+2) \text{Loss}({\text{OPT}}).
$$
\end{proof}

\end{document}